\documentclass[preprint]{article}
\usepackage[nonatbib]{format/neurips_2026}

\usepackage[round]{natbib}

\usepackage[utf8]{inputenc} 
\usepackage[T1]{fontenc}    
\usepackage{hyperref}       
\usepackage{url}            
\usepackage{booktabs}       
\usepackage{amsfonts}       
\usepackage{microtype}      
\usepackage{xcolor}         
\usepackage{colortbl}

\usepackage{amsmath}
\usepackage{amssymb}
\usepackage{mathtools}
\usepackage{amsthm}
\usepackage{graphicx}
\usepackage{bm}
\usepackage{enumitem}
\usepackage{caption}
\usepackage{multirow}
\usepackage{wrapfig}
\usepackage{physics}
\usepackage{siunitx}
\usepackage{subcaption}

\usepackage[capitalize,noabbrev]{cleveref}

\newcommand{\ind}{\mathbf{1}}

\newcommand{\E}{\operatornamewithlimits{\mathbb{E}}}

\theoremstyle{plain}
\newtheorem{theorem}{Theorem}[section]
\newtheorem{proposition}[theorem]{Proposition}

\newtheorem{corollary}[theorem]{Corollary}
\theoremstyle{definition}
\newtheorem{definition}[theorem]{Definition}

\theoremstyle{remark}
\newtheorem{remark}[theorem]{Remark}

\usepackage{tcolorbox}

\tcolorboxenvironment{lemma}{
  colback=black!5,
  colframe=black!10,
}

\usepackage[disable,textsize=tiny]{todonotes}
\setuptodonotes{inline}
\begin{document}
\title{Pandora's Regret: A Proper Scoring Rule for Evaluating Sequential Search}

\author{
  Gerardo A. Flores \\
  MIT \\
  \texttt{nullset@mit.edu} \\
  \And
  Yash Deshpande \\
  Independent \\
  \texttt{yashrajas@gmail.com}
  \AND
  Jannis R. Brea \\
  NYU \\
  \texttt{jannis.brea@nyulangone.org}
  \And
  Ashia C. Wilson \\
  MIT \\
  \texttt{ashia07@mit.edu}
}
\maketitle

\begin{abstract}
In sequential search, alternatives are tested until the true class is found.
Standard proper scoring rules like log loss are local, ignoring the ranking of competitors and misaligning model evaluation with search utility.
We show that sequential search induces a pairwise structure that overcomes this.
By analyzing the expected cost of optimal search under varying testing costs, we derive {\em Pandora’s Regret}: a closed-form, pairwise-additive, and strictly proper scoring rule.
Pandora’s Regret both elicits true probabilities and penalizes rank-reversing miscalibrations where distractors outrank the true class.
Our construction yields a one-parameter Beta family that balances penalties for rank-swapping versus probability magnitude, while retaining a grounded interpretation as expected search cost.
We prove that log loss, accuracy, and macro-F1 rely on implicit decision models misaligned with sequential search.
Across 597 MedMNIST models, Pandora-based metrics better predict clinical diagnostic costs than standard alternatives, extending decision-theoretic scoring rule construction to the multiclass setting.
\end{abstract}

\section{Introduction}
When a patient arrives at the emergency room with abdominal pain, a physician does not test for every condition simultaneously.
Instead, diagnosis proceeds as a \emph{sequential search}: conditions are evaluated in an order determined by their likelihood and the cost of testing.  Similar search problems arise in other domains. An engineer troubleshooting a failing system checks possible faults one by one rather than all at once. An analyst investigating a network outage examines candidate failure points sequentially until the source is found. In each case, alternatives are prioritized by their plausibility and the cost of ruling them out.

The probabilistic forecast derived from a classifier allows the sequencing of tests to be optimized for a particular patient, not just general population guidelines.
Standard machine learning metrics do not model this process.
They allow decision-theoretic interpretations, but correspond to expected costs from decision models where actions are taken simultaneously and independently.
As a result, they do not reward good \emph{within-example ranking} of classes and implicitly assume that decision thresholds can be chosen independently across labels.
This ignores the difference between two forecasts that assign the same probability to the true class but distribute the remaining mass differently across competing classes: one may induce an efficient search order, while the other wastes far more resources testing the wrong alternatives first.

This leaves practitioners with a false choice: use simple but misaligned metrics like Log Loss, or develop bespoke cost models whose complexity precludes their use in routine model selection. We bridge this gap with the following contributions:

\textbf{(1) We pose multiclass evaluation for sequential, cost-sensitive settings as an optimal search problem, within the Pandora framework \citep{weitzman79}.}
Under the optimal policy, classes are evaluated in descending order of probability-to-cost ratio.

\textbf{(2) We derive \emph{Pandora's Regret}, a closed-form, pairwise-additive, strictly proper scoring rule induced by this decision problem.} 
\emph{Pandora's Regret} is the expected excess cost incurred by testing incorrect classes before the true one.
We further characterize a one-parameter family of such scores and use it to connect smoother and more ranking-sensitive regimes.

\textbf{(3) We identify core desiderata for evaluating forecasts in sequential multiclass search: decision alignment, strict propriety, and within-example order sensitivity.}
We show that Pandora's Regret satisfies them, and that log loss, accuracy, and Macro F1 fail for different structural reasons: they correspond to different implicit decision models and therefore miss key features of sequential search.

\textbf{(4) We provide focused empirical evidence on MedMNIST \cite{medmnistv1,medmnistv2} that ranking by Pandora's Regret selects better models than standard alternatives, as measured by downstream diagnostic costs.}
Despite using a uniform, i.i.d.\ cost model, Pandora's Regret still correlates more closely with the task-specific model of diagnostic costs than standard alternatives.

\subsection{Related Work}
\label{sec:related}

\paragraph{Sequential search and clinical decision theory.}
Sequencing tests by probability / cost ratio is a classic idea in economic decision theory and operations research.
Our starting point is the Pandora problem formulation of \cite{weitzman79}, although \cite{matula64} reports Blackwell already knew about the optimal search policy.
Threshold-based analysis for a single diagnosis entered the clinical decision theory literature with \cite{pauker80}, and although \cite{eiseman89} gave a bespoke analysis for two conditions, the core of the sequential decision literature has focused on sequencing different tests for a single diagnosis.
\citet{costeffective08} critiques standard metric evaluation as not accounting for the sequential nature of tests, but remains within the framework of tests for a single condition.
To our knowledge, the $p_k/c_k$ ratio ordering has not been considered in the clinical decision theory literature.
Our contribution is not to advance search theory, but to use the simplifying assumptions of the Pandora problem to give a decision theory-based scoring rule for sequential search with a simple closed-form pairwise structure.

{\bf Proper scoring rules and pairwise proper losses.}
Proper scoring rules incentivize honest probability forecasts and are classically characterized by convex entropy generators \cite{savage71,proper07,dawid07}.
In the binary setting, every proper scoring rule admits an integral representation as a mixture of cost-weighted threshold losses \cite{schervish89}, linking propriety directly to decision structure.
There is a long tradition of using this decomposition as a mixture of simpler decisions to interpret proper scoring rules \cite{murphy66,hand09,recentbeta24}, and \citet{costeffective08} uses this decomposition to argue that evaluation metrics should be evaluated at clinically relevant operating conditions, but remains in the binary setting.

In multiclass settings, \citet{williamson11} show that there cannot be a decomposition onto a simple set of elementary scoring rules.
However, \cite{dawid14} show that sums of scoring rules over marginals of subsets of the variables are themselves proper.
Furthermore, \cite{menon12,menon16} consider learning cross-instance ranking and recalibrating as an approach to training a classifier, but remain within the binary label setting.
Our contribution is therefore not to prove that pairwise proper scoring rules exist, but to give a direct decision-theoretic construction for the Pandora's Regret score and broader Beta family, tying this decomposition to a specific semantics, and empirically evaluate it on a realistic task.

{\bf Label ranking surrogates.}
\cite{wang22} explicitly targets label ranking for multiclass (one-hot) problems and softens pairwise comparisons using sigmoid functions, with or without added Gumbel noise.
More broadly, \cite{wang25} approaches within-instance multilabel ranking (multiple true labels) by asking the probability that a randomly drawn top-K label is higher than that of a randomly drawn positive label.
They then use convex, bounded, decreasing surrogate losses in the difference and prove consistent estimation of the ranking.
\cite{yan22} observe that all such ranking surrogates are translation invariant because scores appear only as differences, so they cannot recover the scale of the underlying labels.
In contrast, the core contribution of our work is the concrete, realistic, discrete sequential search cost grounding and the propriety to elicit the magnitudes of probabilities, not merely their ordering.

\section{The Sequential Search Problem}
\label{sec:sequential}
Many multiclass decision problems are sequential: candidate classes are evaluated one at a time, in an order shaped by both plausibility and the cost of ruling them out.
We formalize this as an optimal search problem and then ask what it means to evaluate probabilistic forecasts for use in such settings.

\begin{figure}[h]
  \centering
  \includegraphics[width=0.65\textwidth]{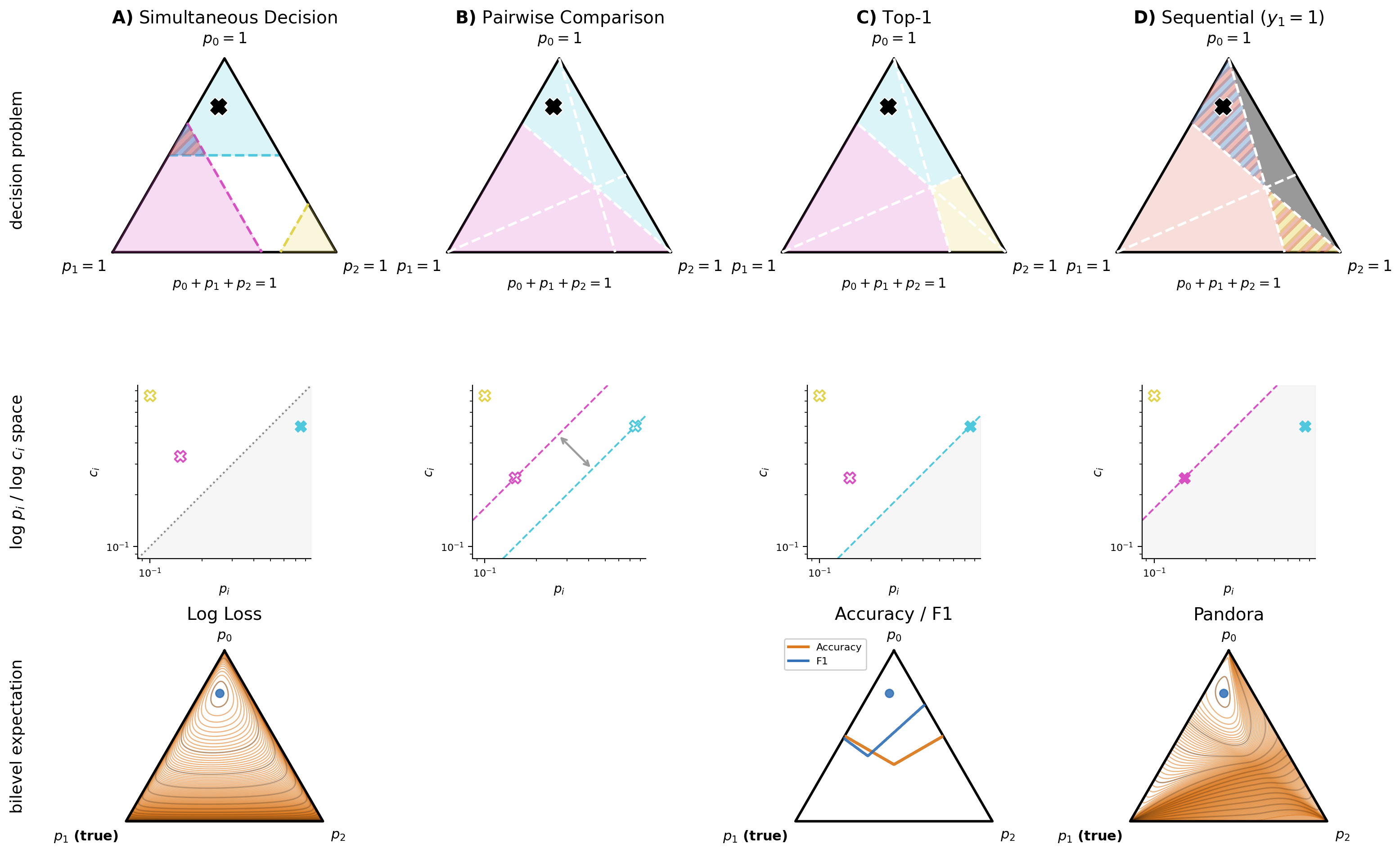}
  \caption{
    A) Simultaneous decision judges each point by whether the $p_k/c_k > 1$.
    In the simplex, some regions test more than one class.
    We plot $p_i,c_i$ separately for each class in the lower diagram.
    Log Loss represents an expectation over possible costs.
    B) A pairwise comparison between two classes splits the simplex in two, and compares the offsets in the $p_i,c_i$ plane.
    Fixing the probability of the true class, the pairwise comparison still changes depending on the ratio of the distractors.
    C) Top-1 is a conjunction of pairwise comparisons; only the lower-right class in the $p_i,c_i$ plane is tested.
    Accuracy and F1 fall in this class with different decision boundaries.
    D) Sequential search starts in the lower right of the $p_i,c_i$ plane, and keeps testing until it reaches the true class.
    In the simplex, some regions test two (striped) or even three (grey) classes.
    The expectation over search costs has more structure than the simultaneous decision.
  }
  \label{fig:simplex_decision_regions}
\end{figure}

\subsection{Optimal Sequential Search}
We begin with the downstream decision problem faced by an individual decision-maker. Let $\vec{p} \in \Delta^{K-1}$ denote the predictive distribution over $K$ classes, and let $c_k > 0$ denote the cost of testing class $k$. We assume that tests are definitive: once the true class is tested, the search ends. This is an example of the \emph{Pandora problem} \cite{weitzman79}, a sequential search problem well-known in economics.

Pandora problems differ from a standard parallel decision view, in which each class is treated as an independent binary action and all classes satisfying a threshold such as $p_k / c_k > 1$ would be tested simultaneously. That formulation ignores the fact that, in many applications, testing is inherently sequential and later tests become unnecessary once the correct class has been found.

Under sequential search, the relevant question is not whether class $k$ should be tested in isolation, but whether it should be tested \emph{before} another class $j$. Consider two classes $k$ and $j$. Testing $k$ before $j$ yields expected total cost $c_k + (1-p_k)c_j$, while testing $j$ before $k$ yields $c_j + (1-p_j)c_k$. The former is preferable exactly when
$
c_k + (1-p_k)c_j \leq c_j + (1-p_j)c_k,
$
which rearranges to
$
\tfrac{p_k}{c_k} \geq \tfrac{p_j}{c_j}.
$

For an instance with true class $i$, the realized search cost is the sum of the costs of all distractors tested before $i$, plus the cost of testing $i$ itself. Equivalently, the cost depends on the position of the true class in the ordering induced by $\{p_k/c_k\}_{k=1}^K$. This already highlights the key evaluative point: two forecasts can assign the same probability to the true class while inducing very different search costs if they rank competing classes differently.

Although we focus here on the simplest setting, the same logic extends to richer decision models with imperfect tests and heterogeneous treatment costs; see \cref{apdx:imperfect_tests}.

\subsection{Bilevel Evaluation}
\label{sec:bilevel}

The sequential search rule describes how an individual clinician should act given predicted probabilities and local testing costs.
Evaluation, however, is a different problem.
The model developer or evaluator does not act on a single patient.
Instead, they choose a model or scoring rule that will be used across many downstream decisions, often under varying operational conditions.
A metric intended to guide model selection should therefore reflect uncertainty not only over patients, but also over the operational cost structure under which predictions will be used.

This leads to a bilevel problem.
The outer level selects a scoring rule; the inner level assumes that clinicians act optimally given predicted probabilities and their local cost structure.
Formally, we evaluate
$$\mathbb{E}_{\vec{c}} \,\mathbb{E}_{(x,y) \in \mathcal{D}_\text{test}} \min_a C(x,y,a;\vec{c})$$
where the expectation over $\vec{c}$ reflects uncertainty in testing costs.
In binary settings, any strictly proper scoring rule is necessarily aligned to some distribution over decision contexts \cite{schervish89}, but in multiclass settings this is no longer generally true  \cite{williamson16}.

\subsection{Desiderata}
\label{sec:desiderata}

The bilevel formulation suggests several properties that a metric for sequential search should satisfy.

\begin{enumerate}[leftmargin=*]

\item \textbf{Decision alignment.}
The metric should reflect the actual downstream decision structure. Here that means alignment with the expected cost of sequential, cost-sensitive search rather than with parallel thresholding, top-1 prediction, or dataset-level aggregate tradeoffs unrelated to the instance at hand.

\item \textbf{Strict propriety.}
Because the downstream user acts on predicted probabilities, the metric should incentivize truthful probabilistic forecasts.
A scoring rule should therefore be minimized in expectation by the true conditional distribution.

\item \textbf{Within-instance order sensitivity.}
In sequential search, the position of the true class in the evaluation order determines costs.
A scoring rule should therefore penalize errors that move the true class down in the evaluation order more highly than similar perturbations of magnitudes that leave the within-instance ranking among competing classes intact.

\end{enumerate}

\section[Pandora's Regret]{Pandora's Regret\footnote{With apologies to connoisseurs of Greek mythology \textit{and} economic search theory.}}

We now derive the scoring rules induced by sequential search. Under broad continuous cost priors, expected search cost is pairwise decomposable and, under weak conditions, strictly proper. Specializing to i.i.d. uniform costs yields a simple closed-form score, \emph{Pandora's Regret}, which we then extend to a $\mathrm{Beta}(\alpha,1)$ family and to heterogeneous class costs. The resulting scores place extra weight on the within-instance miscalibrations that matter for sequential search, especially those that reverse the ordering of classes.
\subsection{Expected Search Cost is a Pairwise Decomposable, Proper Scoring Rule}

We show that expected search cost decomposes into pairwise comparisons between the true class and each distractor.

\begin{theorem}[Search cost decomposition]
\label{thm:decomposition_body}
Let $\vec{p} \in \mathrm{int}(\Delta^{K-1})$ and let $i$ be the true class. Let testing costs $\vec{c}=(c_1,\dots,c_K)$ be drawn from a distribution $F$ on $\mathbb{R}_{>0}^K$ such that $\mathbb{E}[c_k] < \infty$ for all $k$, and assume $\vec{c} \perp y$. Under the optimal sequential search policy, which tests classes in decreasing order of $p_k/c_k$, the expected total search cost decomposes as
\begin{equation}
S_F(\vec{p},y{=}i)=\mathbb{E}_F[c_i]+\sum_{j\neq i} L_F^{i\to j}\!\left(p_i/p_j\right),
\end{equation}
where $ L_F^{i\to j}(r):= \mathbb{E}_F\!\left[c_j\,\ind\!\left(c_j < \frac{c_i}{r}\right)\right]$
depends on $F$ only through the bivariate marginal of $(c_i,c_j)$. Moreover:
\begin{enumerate}
\item[\textup{(i)}] $S_F$ is proper.
\item[\textup{(ii)}] $S_F$ is strictly proper if, for every $i\neq j$, the ratio $c_i/c_j$ has support $(0,\infty)$.
\end{enumerate}
\end{theorem}

Full proof details are in \cref{thm:decomposition} and \cref{thm:properness}. The decomposition follows directly from the ratio rule. If the true class is $i$, then the search process always incurs the cost $c_i$ of testing the correct class. In addition, it incurs the cost $c_j$ of a distractor $j$ exactly when $j$ is tested before $i$. Under the optimal policy, this happens if and only if
\[
{p_j/c_j} > {p_i/c_i}
\quad\Longleftrightarrow\quad
c_j < c_i/(p_i/p_j),
\]
Taking expectations therefore yields a sum of pairwise terms, one for each distractor.
This decomposition is important for two reasons. First, it shows that the score depends on $\vec p$ through the relative ordering of the true class and each competitor, rather than only through the probability assigned to the true class. Second, it shows that sequential search naturally gives rise to a \emph{pairwise-additive} scoring rule, with each pairwise term indexed by the odds ratio $p_i/p_j$.

Propriety follows because, for any realized cost vector, the true conditional distribution induces the cost-minimizing ordering in expectation. Averaging over cost realizations preserves this property. Strict propriety follows under full support of the pairwise cost ratios: if the forecast changes any ratio $p_i/p_j$, then there is a positive-measure set of cost realizations for which the induced ordering of $i$ and $j$ is suboptimal. Thus any deviation from the truth incurs strictly greater expected search cost.

Theorem~\ref{thm:decomposition_body} already gives a broad class of proper scoring rules. We now specialize to a particularly simple and interpretable choice of cost prior.
To obtain a concrete closed-form score, we consider the canonical bounded prior in which testing costs are drawn independently from the uniform distribution on $[0,1]$. This prior yields a simple pairwise loss that, after removing constants and normalizing, produces our main scoring rule.

\begin{definition}[Pandora's Regret]
\label{def:pandora_regret}
Let $c_i,c_j \stackrel{\mathrm{iid}}{\sim} \mathrm{Unif}[0,1]$. Define the pairwise loss
\begin{equation}
L_{\mathrm{pandora}}(r) :=
6\,\mathbb{E}[c_j\,\ind(c_j < c_i/r)]
=
\begin{cases}
3 - 2r, & r \leq 1, \\[4pt]
r^{-2}, & r > 1.
\end{cases}
\end{equation}
Then \emph{Pandora's Regret} is
\begin{equation}
S_{\mathrm{pandora}}(\vec p,i) := \frac{1}{3(K-1)}
\sum_{j\neq i}
L_{\mathrm{pandora}}\!\left(\tfrac{p_i}{p_j}\right).
\end{equation}
\end{definition}
\vspace{-1em} 
The pairwise term has an intuitive shape. When $r=p_i/p_j \leq 1$, the distractor $j$ is at least as probable as the true class, so the loss decreases linearly as the true class catches up. When $r>1$, the distractor already lies below the true class, and the penalty decays quadratically. Thus Pandora's Regret is especially sensitive to the rank reversals and near-ties that matter for sequential search.

\begin{theorem}[Uniform-cost form and strict propriety]
\label{thm:pandora_regret_proper}
Let $c_1,\dots,c_K \stackrel{\mathrm{iid}}{\sim} \mathrm{Unif}[0,1]$, independent of the label $y$. Then the expected optimal search cost satisfies
\begin{equation}
S_{\mathrm{unif}}(\vec p,y{=}i)
=
\frac{1 + (K-1)S_{\mathrm{pandora}}(\vec p,i)}{2}.
\end{equation}
Consequently, $S_{\mathrm{pandora}}$ is a strictly proper scoring rule.
\end{theorem}

Theorem~\ref{thm:pandora_regret_proper} shows that Pandora's Regret is not merely motivated by sequential search; it is an affine transformation of expected optimal search cost under the uniform prior. The normalization makes the score easier to compare across examples while preserving its decision-theoretic interpretation. In particular, $(K-1)S_{\mathrm{pandora}}(\vec p,i)=x$ implies an $\tfrac{x}{100}\%$ expected excess search burden relative to oracle performance.

This construction satisfies the desiderata of \cref{sec:desiderata}. It is {\em strictly proper}, because it is uniquely minimized by the true conditional distribution. It is {\em order-sensitive}, because it decomposes into pairwise comparisons with each distractor and changes sharply when the true class falls below a competitor. It is {\em decision-aligned}, because it is derived directly from expected search cost under the sequential-search problem. And it is differentiable and bounded, making it suitable not only for model comparison but also, in principle, for recalibration and optimization.

\subsection{A Beta family of search-based scoring rules}
\label{sec:beta_family}

\begin{figure}[h]
  \centering
  \includegraphics[width=0.65\textwidth]{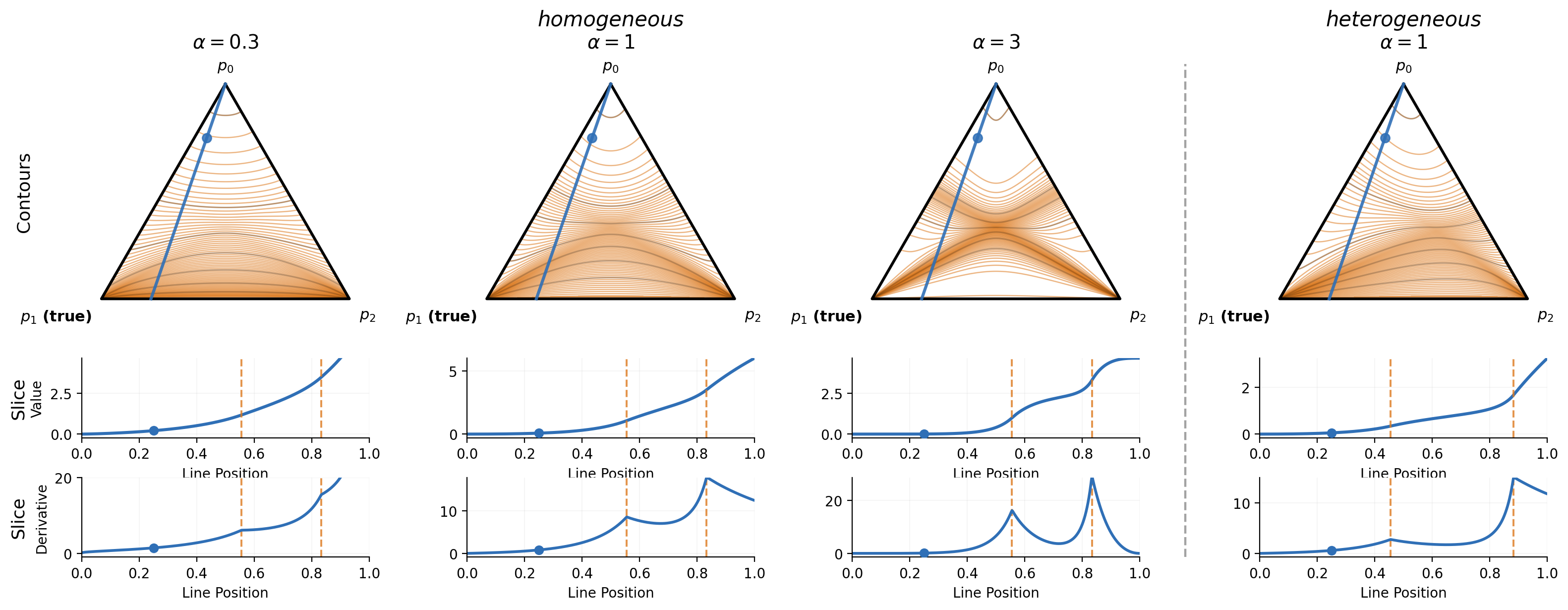}
  \caption{
Pandora's Regret is the $\alpha=1$ unit-cost member of a one-parameter $\mathrm{Beta}(\alpha,1)$ family of pairwise-additive scores, strictly proper for every finite $\alpha>0$.
As $\alpha\to\infty$, the score becomes increasingly rank-focused, and its limit depends only on the weighted ordering of the ratios $p_k/C_k$.
As $\alpha\downarrow0$, the pairwise loss tends to $1-\ln r$ for $r\le1$ and $r^{-1}$ for $r>1$.
When base costs $C_1,\dots,C_K>0$ are available, the weighted Beta score preserves strict propriety for every finite $\alpha>0$.
  }
  \label{fig:pairwise_alpha_simplex}
\end{figure}

Pandora's Regret is the $\alpha=1$ unit-cost specialization of a broader $S^{\mathrm{Beta}}_{\alpha,\mathbf C}$ family obtained by drawing latent unit costs from $\mathrm{Beta}(\alpha,1)$.
For $\alpha>0$, define
\[
  L_\alpha(r)
  \;=\;
  \begin{cases}
    1 + \tfrac{\alpha+1}{\alpha}\left(1-r^{\alpha}\right)
     & r \le 1, \\[4pt]
    r^{-(\alpha+1)}
      & r > 1
  \end{cases}
\]
In the appendix, we work with the normalized score $S^{\mathrm{Beta}}_{\alpha,\mathbf C}$ and its exact affine relation to raw expected search cost.
Here, for readability, we use the rescaled unit-cost score
\(
  S_\alpha(\vec p,i)
  :=
  \frac{1}{3(K-1)}\sum_{j\neq i} L_\alpha(p_i/p_j),
\)
so that $S_1 = S_{\mathrm{Pandora}}$.
See \cref{fig:pairwise_alpha_simplex} for a visualization of the score for several values of $\alpha$.
See \cref{apdx:beta_family,cor:beta_one,cor:beta_zero,cor:beta_infty} for the derivations and endpoint limits.

\paragraph{Heterogeneous class costs.}
If deterministic base costs $C_1,\dots,C_K>0$ are available, define the cost-adjusted probabilities $q_k:=p_k/C_k$.
The appendix's weighted Beta score is
\[
  S^{\mathrm{Beta}}_{\alpha,\mathbf C}(\vec p,i)
  \;:=\;
  \sum_{j\neq i}
  C_j\,L_\alpha\!\left(\frac{q_i}{q_j}\right).
\]
Dividing by $K-1$ simply averages over distractors and does not affect propriety.
This corresponds to realized testing costs of the form $c_k=C_k u_k$ with $u_k\stackrel{\mathrm{iid}}{\sim}\mathrm{Beta}(\alpha,1)$, so the optimal policy still orders classes by probability-to-cost ratio.
Strict propriety is preserved for every finite $\alpha>0$ and every $C_k>0$
(\cref{cor:beta_weighted}).

Taken together, these results show that Pandora's Regret is one member of a broader family of search-based scoring rules.
For every finite $\alpha>0$, the Beta score is strictly proper, and the limits $\alpha\downarrow0$ and $\alpha\to\infty$ are interesting.
We next compare this decision-aligned family to the standard multiclass metrics most commonly used in practice: log loss, accuracy, and macro-F1.

\section{Standard Metrics as Decision Models}

To clarify what Pandora's Regret changes, we analyze standard multiclass metrics through the decision problems they implicitly define. These fall into three categories: strictly proper, top-1, and full ranking metrics. Each captures a different notion of predictive quality, but none aligns with a sequential, cost-sensitive search. We recover the underlying decision structure from each metric's expected cost and analyze a representative example from each class.

\subsection{Strictly Proper: Log Loss}

Log loss is strictly proper, widely used, and admits classical decision-theoretic interpretations in the binary setting \cite{shuford66,savage71,schervish89}.
We show that log loss admits two interpretations as a bilevel optimization of a multiclass decision problem, one parallel and one sequential, but in both cases the resulting score depends only on the probability assigned to the true class.
Thus its insensitivity to within-instance ranking is structural.

\subsubsection{Log Loss as Parallel Decisions}
One interpretation of multiclass log loss is as a collection of $K$ independent binary decisions. For each class $j$, the decision-maker may treat at cost $c_j$ or incur cost $1$ if $j$ is in fact the true class.

\begin{theorem}[Cross-Entropy as Parallel Decisions]
\label{thm:logloss_structure}
Let $\vec p \in \Delta^{K-1}$ and $i$ be the true class.
The optimal policy is to treat if and only if $p_j/c_j > 1$.
If costs are drawn independently with density $w(c) \propto 1/c$ on $(0,1)$,
the expected optimal cost is
\(
S(\vec p,i) = -\ln p_i + 1.
\)
\end{theorem}

The proof is in Appendix~\ref{apdx:logloss_search}.
The treatment costs in the parallel decision model sum to $1$ regardless of the prediction, so log loss compares forecasts only through undertreatment of the true class.

\subsubsection{Log Loss as Fixed-Order Sequential Search}
\label{apdx:ce_sequential}

A sequential reinterpretation exists, but requires fixing the search order across all cost realizations.

\begin{proposition}[Cross-Entropy as Fixed-Order Sequential Search]
\label{thm:sequential_logloss}
Assume a fixed search order $\sigma$.
If testing costs are drawn independently from the Haldane measure $w(c) \propto \frac{1}{c(1-c)}$, and the fixed miss penalty is $1$, the expected search cost until the true class $m$ is found is
\(S_{\mathrm{seq}}(\vec{p}, m) \;=\; -\ln p_m \).
\end{proposition}

Under this interpretation, the search suffers overtreatment regret for each class examined before the truth, and then undertreatment regret for stopping at the true class. Algebraically, the intermediate terms reduce to log conditional survival probabilities, while the true step contributes the log conditional probability of stopping. These terms telescope, leaving only the log probability of the true class (see \cref{apdx:cost_amnesia_proof} for the full proof).

Thus even in a step-by-step search model, log loss remembers only the probability assigned to the true class and perfectly forgets the economic penalties of early wrong turns. As detailed in \cref{rmk:haldane_critical}, this behavior is extremely fragile: it strictly forbids cost-adaptive ordering and relies entirely on the knife-edge Haldane measure to function.

\subsection{Order Sensitive: Top-1 Measures}

Accuracy, the most popular measure, evaluates whether $\arg\max_k p_k = i$.
Its implicit decision model is a single-shot top-1 choice: select one class and incur unit cost if that choice is wrong.
This is useful in some settings, but in most diagnostic settings, clinical or otherwise, the search will not terminate after one failure, so accuracy fails to capture most of the relevant information about order.
Moreover, it is not strictly proper as it is flat almost everywhere, and so cannot incentivize correct probability forecasts.
Details for all these statements appear in Appendix~\ref{apdx:accuracy_proofs}.

Macro-F1, the unweighted average of per-class F1 scores, while very popular, is not usually presented as a decision model.
However, building on the local gradient characterization of \citet{narasimhan15}, the classification incentives for the marginal patient can be interpreted as a classwise decision rule with adaptive rewards and penalties that depend on the rest of the patients.
These do not, however, align with the costs for the individual patient being diagnosed.
Thus it inherits the minimal order sensitivity and lack of strict propriety of accuracy, to which it adds poorer decision alignment and a lack even of propriety.
A formal statement and proof appear in Appendix~\ref{apdx:f1_proofs}.
\paragraph{Summary.}
Among the standard alternatives, log loss is the strongest: it is proper and admits clear decision-theoretic interpretations. But both of those interpretations erase the within-instance ordering information that matters for sequential search. Accuracy is simpler still, reducing prediction to a single top-1 decision and discarding all information about what should happen after an initial mistake. Macro-F1 is the least compatible with our setting because it violates instance-level independence and evaluates one example partly through its interaction with the rest of the dataset. These contrasts clarify the role of Pandora's Regret: it is not merely another metric with different weights, but a scoring rule derived from a different decision model.

\section{Empirical Validation}
\label{sec:empirical}

Pandora's Regret aligns well with our theoretical desiderata, but we also want to know how well it chooses among real models on real data.
We operationalize this question using medical image classification tasks for which we can write a simple task-specific diagnostic cost model.
We then rank a zoo of nearly six hundred image models by the task-specific diagnostic costs they would produce if used to choose a testing plan.
We rerank as well by standard metrics, and compare those rankings using the absolute value of Kendall's $\tau$ (some metrics are low for good models, and others high).

{\bf Model Zoo.}\quad
Because the purpose is to understand the correlation between metrics' rankings of model quality rather than to set state of the art, we use nearly six hundred pretrained models from the PyTorch Image Models (TIMM) library \citep{wightman19} that accept 224x224 inputs and have fewer than 50M parameters (to keep our runtime within one GPU day), with a linear probe head trained based on cross entropy (rather than Pandora's Regret) on our datasets for 1 epoch.  Training details are the same for all models, and provided in \cref{apdx:training}.  However, we do not report the ranking of this zoo on the downstream tasks, but rather its correlation with the rankings of the models under Pandora's Regret and the various standard metrics.

{\bf MedMNIST Data.}\quad
We use the two imaging datasets from the MedMNIST benchmark for which good computer vision performance is established, and for which there are more than two classes that map to clinically distinct treatment actions: DermaMNIST (7 skin lesion classes) and OCTMNIST (4 retinal conditions).  We use the 224x224 originals rather than the 28x28 downsampled versions for more reasonable quality, train on the training split for 1 epoch, and evaluate on the held-out test split.

\subsection{Task Specific Diagnostic Cost Model}

For each class, we assign a diagnostic cost reflecting the real clinical workup triggered by suspicion of that condition (e.g., melanoma requires biopsy plus histopathology plus immunohistochemistry at \$383.77; benign conditions require only a specialist visit at \$75.15; see \cref{app:costs} and \cref{tab:cost_vector}).
Given a model's costs and predicted probabilities, we simulate an optimal sequential search that tests hypotheses in descending order of probability-to-cost ratio, accumulating costs until the true condition is identified.
We use the perfect-test, zero-treatment-cost simplification; \cref{apdx:imperfect_tests} shows that relaxing these assumptions modifies effective per-class costs but preserves the same order structure.
\setlength{\floatsep}{3pt}
\setlength{\intextsep}{3pt}
\setlength{\textfloatsep}{3pt}

\begin{table}[!htbp]
  \centering
  \begin{tabular}{l l c c c c}
    \toprule
    Target & Dataset & $S_{\text{Pandora}}$ \textbf{i.i.d.} & Log Loss & Accuracy & F1 (macro) \\
    \midrule
    \textbf{Clinical} & Derm & \textbf{0.86} & 0.84 & 0.77 & 0.74 \\
     & OCT & \textbf{0.90} & 0.88 & 0.87 & 0.88 \\
    \midrule
    \textbf{Well Specified} & Derm & \textbf{0.93} & 0.90 & 0.84 & 0.79 \\
     & OCT & \textbf{0.96} & 0.91 & 0.84 & 0.86 \\
    \midrule
    \textbf{Random} & Derm & \textbf{0.82} & 0.65 & 0.62 & 0.62 \\
    \textbf{Temperature} & OCT & \textbf{0.87} & 0.71 & 0.85 & 0.86 \\
    \midrule
    \textbf{Random} & Derm & \textbf{0.75} & 0.19 & 0.38 & 0.17 \\
    \textbf{Neg. Temp.} & OCT & \textbf{0.87} & 0.84 & 0.72 & 0.73 \\
    \bottomrule
    \end{tabular}
  \caption{
    Absolute value of Kendall's $\tau$ between model rankings under each metric (column) and task-specific simulated search cost.
    Conditions are described in \cref{sec:meta_eval}. 
    See \cref{fig:tau_ci} for confidence intervals.
}
  \label{tab:tau}
  \small
\end{table}
\subsection{Model Ranking Meta-evaluation}
\label{sec:meta_eval}
We compare the performance of Pandora's Regret, Log Loss, Accuracy, and Macro F1 on the DermaMNIST and OCTMNIST datasets across four conditions on costs and predictions. Across all conditions, Pandora’s Regret yields rankings that more closely align with downstream search cost than standard metrics, with the gap widening when assumptions match the model or when distractor distributions are perturbed.

\begin{itemize}[leftmargin=*]
  \item \textbf{Clinical}: Uses the task-specific diagnostic cost model.
  \item \textbf{Well-Specified}: Draws class costs from $\mathrm{Unif}[0,1]$ to match the model assumptions. Pandora’s Regret shows a larger advantage, as expected.
  \item \textbf{Random Temperature}: Applies lognormal temperature scaling to vary prediction confidence. Pandora’s Regret maintains a stronger lead over log loss, though accuracy and F1 remain competitive on OCTMNIST.
  \item \textbf{Random Distractor Temperature}: Rescales only distractor classes to vary their distribution while preserving true-class calibration, isolating sensitivity to negative-class probabilities.
\end{itemize}

\begin{figure}[!htbp]
  \centering
  \includegraphics[width=\textwidth]{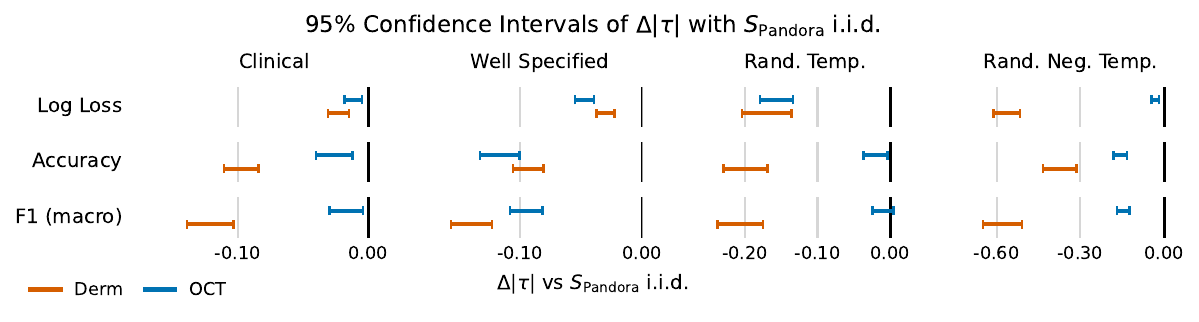}
  \caption{
    Confidence intervals for the gap $\Delta|\tau|$ between each metric and $S_\text{Pandora}$.
    Conditions (columns) are described in \cref{sec:meta_eval}.
    The confidence interval for F1 on random temperature rescalings of OCTMNIST predictions overlaps zero; all others are bounded below zero.
    See \cref{tab:tau} for point estimates of Kendall's $\tau$ itself.
  }
  \label{fig:tau_ci}
\end{figure}

\setlength{\floatsep}{12pt plus 2pt minus 2pt}
\setlength{\intextsep}{12pt plus 2pt minus 2pt}
\setlength{\textfloatsep}{20pt plus 2pt minus 4pt}
This is consistent with the theoretical analysis: metrics that incorporate within-instance ordering better capture the cost of early mistakes in sequential search.

\section{Conclusion, Limitations, and Future Work}
\label{sec:conclusion}
Evaluation criteria are products of the decisions they support.
In sequential settings, where costs accrue as alternatives are ruled out, forecast quality depends not just on the probability assigned to the true class, but on how competing alternatives are ordered.
We formalize this as an optimal search problem and derive a corresponding class of scoring rules.
The resulting score, Pandora’s Regret, is strictly proper and pairwise-additive, and directly reflects expected excess search cost.
The associated Beta family shows that even within this decision model, evaluation spans a continuum between sensitivity to probability magnitudes and to within-instance ranking.

This perspective also clarifies standard metrics.
Log loss, accuracy, and macro-F1 correspond to distinct implicit decision models (parallel decisions, single-shot selection, and population-level tradeoffs) and their failures in sequential, cost-sensitive settings are therefore structural.
Practically, this suggests that when decisions are sequential and verification is heterogeneous, evaluation should reward forecasts that induce efficient search orderings, not just correct final predictions.
Pandora’s Regret provides a tractable instance of this principle, but the broader implication is methodological: evaluation metrics should be derived from the decision processes they are intended to approximate.

{\bf Limitations.} Our analysis relies on a stylized sequential search model with perfect tests and simplified costs. While many extensions can be absorbed into effective per-class costs, richer settings such as tests that inform multiple hypotheses or decisions with path-dependent utilities may not admit a pairwise additive form, trading tractability for expressiveness. Empirically, we evaluate on two MedMNIST tasks with approximate diagnostic costs. These provide initial evidence of decision alignment, but do not capture the variability of real-world settings, including patient-specific costs, institutional differences, and complex workflows. Broader validation across domains and cost structures is needed to assess robustness.

{\bf Future Directions.} An important direction is extending this framework to richer decision settings, including hierarchical testing, partially informative actions, and adaptive decision processes. The challenge is to retain propriety and tractability in these more complex environments. More broadly, this work suggests a shift from selecting generic metrics to deriving evaluation criteria from decision structure. Sequential search provides one tractable case. Understanding how far this approach extends, and how it interacts with fairness, robustness, and deployment constraints, remains an open question.

\bibliographystyle{plainnat}
\bibliography{references,empirical/apdx_costs,alternatives/apdx_f1}

\newpage
\appendix
\onecolumn

\section{Search-Based Scoring Rules}
\label{apdx:search_scoring}

The appendix formalizes the search-based view of scoring used in the main text.
Given a forecast $\vec p$ and realized testing costs $\vec c$, the optimal policy searches classes in decreasing order of $p_k/c_k$.
For any pair of classes $i$ and $j$, this reduces to a single threshold comparison: does the realized cost ratio $c_i/c_j$ exceed the forecast odds ratio $p_i/p_j$?

This ratio-vs.-ratio view drives everything that follows.
First, it yields a clean pairwise decomposition of the total expected search cost into one distractor term for each class pair (\cref{apdx:pairwise_decomposition}).
Second, it makes propriety essentially immediate: under the true label distribution $\vec\pi$, the Bayes-optimal threshold for pair $\{i,j\}$ is exactly $\pi_i/\pi_j$, so each pairwise term is minimized when the forecast odds match the true odds (\cref{apdx:properness}).
Third, deterministic class costs enter only through a rescaling of the odds, from $p_i/p_j$ to $(p_i/C_i)/(p_j/C_j)$, without changing the underlying mechanism (\cref{apdx:known_cost}).
Finally, specializing the unit costs to i.i.d.\ $\mathrm{Beta}(\alpha,1)$ draws yields the closed-form family used in the main text, with $\alpha$ interpolating between smoother magnitude-sensitive penalties and increasingly rank-like behavior (\cref{apdx:beta_family}).

\subsection{Pairwise decomposition}
\label{apdx:pairwise_decomposition}

Given realized costs $\vec c=(c_1,\dots,c_K)$, the optimal search policy
inspects classes in decreasing order of $p_k/c_k$.
Ties are broken uniformly at random, independently of $(y,\vec c)$.
The appendix uses one observation repeatedly:
for any pair $i\neq j$,
\[
  r_{ij}:=\frac{p_i}{p_j},
  \qquad
  \tau_{ij}:=\frac{c_i}{c_j},
\]
and class $j$ is searched before class $i$ exactly when
$\tau_{ij}>r_{ij}$, with random tie-breaking when $\tau_{ij}=r_{ij}$.
Thus the multiclass search cost decomposes into pairwise distractor terms,
and properness reduces to the fact that under true label probabilities
$\vec\pi$, the Bayes-optimal threshold for pair $\{i,j\}$ is
$\pi_i/\pi_j$.

It is convenient to write
\[
  h(a,b)
  \;:=\;
  \ind(a>b)+\tfrac12\ind(a=b),
\]
so that $h(a,b)$ is the probability that the first quantity beats the second
under this tie-breaking rule.

\begin{theorem}[Search cost decomposition]
\label{thm:decomposition}
Let $\vec p\in\mathrm{int}(\Delta^{K-1})$, let $i$ be the true class, and
let $\vec c=(c_1,\dots,c_K)$ be drawn from a distribution $F$ on
$\mathbb R_{>0}^K$ with $\E_F[c_k]<\infty$ for all $k$.
Under the optimal sequential search policy,
\begin{equation}
\label{eq:search_cost}
  S_F(\vec p,y{=}i)
  \;=\;
  \E_F[c_i]
  \;+\;
  \sum_{j\neq i}
  L_F^{i\to j}\!\left(\frac{p_i}{p_j}\right),
\end{equation}
where the directed pairwise loss is
\begin{equation}
\label{eq:directed_pairwise_loss}
  L_F^{i\to j}(r)
  \;:=\;
  \E_F\!\left[
    c_j\,h\!\left(\frac{c_i}{c_j},\,r\right)
  \right].
\end{equation}
Equivalently,
\[
  L_F^{i\to j}(r)
  \;=\;
  \E_F\!\left[
    c_j\!\left(
      \ind\!\left(\frac{c_i}{c_j}>r\right)
      +
      \tfrac12\ind\!\left(\frac{c_i}{c_j}=r\right)
    \right)
  \right].
\]
Moreover, $L_F^{i\to j}$ depends on $F$ only through the joint law of
$(c_i,c_j)$.
\end{theorem}

\begin{proof}
Fix a realized cost vector $\vec c$.
If the true class is $i$, the searcher always pays the unavoidable cost
$c_i$.
Each distractor $j\neq i$ contributes an additional cost $c_j$ exactly when
$j$ is searched before $i$.

For the pair $(i,j)$, let
\[
  r_{ij}=\frac{p_i}{p_j},
  \qquad
  \tau_{ij}=\frac{c_i}{c_j}.
\]
Since classes are searched in decreasing order of $p_k/c_k$,
\[
  \frac{p_j}{c_j}>\frac{p_i}{c_i}
  \quad\Longleftrightarrow\quad
  \frac{c_i}{c_j}>\frac{p_i}{p_j}
  \quad\Longleftrightarrow\quad
  \tau_{ij}>r_{ij}.
\]
On the tie event $\tau_{ij}=r_{ij}$, the tied-block rule places $j$ before
$i$ with probability $1/2$.
Hence the realized total search cost is
\[
  c_i
  +
  \sum_{j\neq i}
  c_j\,h(\tau_{ij},r_{ij}).
\]
Taking expectations gives \cref{eq:search_cost}.
Each summand involves only $(c_i,c_j)$, so $L_F^{i\to j}$ depends only on
their joint law.
\end{proof}

\subsection{Properness and strict propriety}
\label{apdx:properness}

Properness is now a pairwise statement.
For each unordered pair $\{i,j\}$, the score chooses a search order using
the threshold $p_i/p_j$, and the Bayes-optimal threshold turns out to be the
true odds ratio $\pi_i/\pi_j$.

\begin{theorem}[Properness and strict propriety]
\label{thm:properness}
Let $\vec p\in\mathrm{int}(\Delta^{K-1})$, let
$\vec c=(c_1,\dots,c_K)$ be drawn from a distribution $F$ on
$\mathbb R_{>0}^K$ with $\E_F[c_k]<\infty$ for all $k$, and assume
$\vec c$ is independent of $y$.
Then $S_F$ is proper.

It is strictly proper if, for every $i\neq j$, the ratio $c_i/c_j$ has full
support on $(0,\infty)$, meaning that every nonempty open interval in
$(0,\infty)$ has positive probability.
\end{theorem}

\begin{proof}
Let $\vec\pi\in\mathrm{int}(\Delta^{K-1})$ be the true conditional label
distribution, and define the Bayes risk
\[
  \mathcal R(\vec p;\vec\pi)
  \;:=\;
  \sum_{k=1}^K \pi_k\,S_F(\vec p,y{=}k).
\]
By \cref{thm:decomposition},
\[
  \mathcal R(\vec p;\vec\pi)
  \;=\;
  \sum_{k=1}^K \pi_k\,\E_F[c_k]
  \;+\;
  \sum_{i<j}
  \Phi_{ij}\!\left(\frac{p_i}{p_j}\right),
\]
where
\[
  \Phi_{ij}(r)
  \;:=\;
  \E_F\!\left[
    \pi_i c_j\,h\!\left(\frac{c_i}{c_j},\,r\right)
    +
    \pi_j c_i\,h\!\left(r,\,\frac{c_i}{c_j}\right)
  \right].
\]
The first term is independent of $\vec p$, so it is enough to minimize each
pairwise term $\Phi_{ij}$.

Fix a pair $\{i,j\}$ and write
\[
  r:=\frac{p_i}{p_j},
  \qquad
  r^\star:=\frac{\pi_i}{\pi_j},
  \qquad
  \tau:=\frac{c_i}{c_j}.
\]
For fixed realized costs $(c_i,c_j)$, there are only two relevant search
orders for this pair:

\medskip
\noindent
(i) search $j$ before $i$, which contributes $\pi_i c_j$ to the conditional
Bayes risk;

\noindent
(ii) search $i$ before $j$, which contributes $\pi_j c_i$.

\medskip
Hence the better order is
\[
  j \prec i
  \quad\Longleftrightarrow\quad
  \pi_i c_j \le \pi_j c_i
  \quad\Longleftrightarrow\quad
  \frac{c_i}{c_j} \ge \frac{\pi_i}{\pi_j}
  \quad\Longleftrightarrow\quad
  \tau \ge r^\star.
\]
So the Bayes-optimal threshold for the pair is exactly $r^\star$:
search $j$ before $i$ when $\tau>r^\star$, search $i$ before $j$ when
$\tau<r^\star$, and either order is optimal on the tie event
$\tau=r^\star$.

But the score with threshold $r$ implements precisely this threshold rule:
it searches $j$ before $i$ when $\tau>r$, searches $i$ before $j$ when
$\tau<r$, and randomizes on $\tau=r$.
Therefore, for every realized $(c_i,c_j)$,
\[
  \pi_i c_j\,h(\tau,r)+\pi_j c_i\,h(r,\tau)
  \;\ge\;
  \pi_i c_j\,h(\tau,r^\star)+\pi_j c_i\,h(r^\star,\tau).
\]
Taking expectations gives
\[
  \Phi_{ij}(r)\ge \Phi_{ij}(r^\star).
\]
Thus each pairwise Bayes-risk term is minimized at
$r^\star=\pi_i/\pi_j$.

Since this holds for every unordered pair,
$\mathcal R(\vec p;\vec\pi)$ is minimized whenever
\[
  \frac{p_i}{p_j}=\frac{\pi_i}{\pi_j}
  \qquad\text{for all } i\neq j.
\]
On $\mathrm{int}(\Delta^{K-1})$, these pairwise ratio constraints are
equivalent to $\vec p=\vec\pi$.
Hence $S_F$ is proper.

For strict propriety, fix a pair $\{i,j\}$ and suppose $r\neq r^\star$.
Then the open interval between $r$ and $r^\star$ is nonempty.
If $c_i/c_j$ has full support on $(0,\infty)$, that interval has positive
probability.
On that event, the threshold rule at $r$ chooses the wrong order relative to
the Bayes-optimal threshold $r^\star$, so the excess pairwise cost is
strictly positive:
\[
  \bigl|\pi_j c_i-\pi_i c_j\bigr|
  \;=\;
  c_j\,\bigl|\pi_j\tau-\pi_i\bigr|
  \;>\; 0.
\]
Hence $\Phi_{ij}(r)>\Phi_{ij}(r^\star)$ whenever $r\neq r^\star$.
So every pairwise term is uniquely minimized at the true odds, and the full
Bayes risk is uniquely minimized at $\vec p=\vec\pi$.
Therefore $S_F$ is strictly proper.
\end{proof}

\begin{remark}[Finitely supported empirical priors]
A finitely supported empirical or bootstrapped prior over costs is fully
compatible with the construction above.
Such priors are generally proper but not strictly proper:
only finitely many pairwise threshold values can occur, so the expected
score is typically piecewise flat as a function of the forecast odds.
\end{remark}

\subsection{Known-cost rescaling}
\label{apdx:known_cost}

Deterministic class costs do not change the underlying unit-cost mechanism;
they only rescale the relevant odds.
This is the form used below to derive the weighted Beta family.

\begin{proposition}[Known-cost rescaling]
\label{prop:known_cost_rescaling}
Let $c_k=C_k u_k$ with $C_k>0$ fixed and
$(u_1,\dots,u_K)$ drawn from a joint distribution satisfying
$\E[u_k]<\infty$ for all $k$.
Define the cost-adjusted scores
\[
  s_k:=\frac{p_k}{C_k},
\]
and let $L_u^{i\to j}$ denote the directed pairwise loss induced by the
unit-cost process $(u_1,\dots,u_K)$ under the same tie-breaking convention
as above.
Then
\begin{equation}
\label{eq:known_cost_score}
  S(\vec p,y{=}i)
  \;=\;
  C_i\,\E[u_i]
  \;+\;
  \sum_{j\neq i}
  C_j\,L_u^{i\to j}\!\left(\frac{s_i}{s_j}\right).
\end{equation}
Moreover, if $u_i/u_j$ has full support on $(0,\infty)$ for every
$i\neq j$, then the weighted score is strictly proper.
\end{proposition}

\begin{proof}
Since $c_k=C_k u_k$,
\[
  \frac{p_k}{c_k}
  \;=\;
  \frac{p_k/C_k}{u_k}
  \;=\;
  \frac{s_k}{u_k}.
\]
So the search order under realized costs $(c_1,\dots,c_K)$ is exactly the
search order induced by the cost-adjusted scores $(s_1,\dots,s_K)$ and the
unit costs $(u_1,\dots,u_K)$.

Fix the true class $i$.
The unavoidable term is $c_i=C_i u_i$, whose expectation is
$C_i\,\E[u_i]$.
For each distractor $j\neq i$, the extra cost is $c_j=C_j u_j$, and it is
incurred exactly when $j$ is searched before $i$, i.e.
\[
  \frac{s_j}{u_j}>\frac{s_i}{u_i}
  \quad\Longleftrightarrow\quad
  \frac{u_i}{u_j}>\frac{s_i}{s_j},
\]
again with random tie-breaking on equality.
Therefore the expected contribution of distractor $j$ is
\[
  C_j\,\E\!\left[
    u_j\,h\!\left(\frac{u_i}{u_j},\,\frac{s_i}{s_j}\right)
  \right]
  \;=\;
  C_j\,L_u^{i\to j}\!\left(\frac{s_i}{s_j}\right).
\]
Summing over $j\neq i$ gives \cref{eq:known_cost_score}.

Finally, if $u_i/u_j$ has full support on $(0,\infty)$, then so does
\[
  \frac{c_i}{c_j}
  \;=\;
  \frac{C_i}{C_j}\,\frac{u_i}{u_j}.
\]
Hence \cref{thm:properness} applies directly to the cost process
$(c_1,\dots,c_K)$, so the weighted score is strictly proper.
\end{proof}

\subsection{A Beta family of search-based scoring rules}
\label{apdx:beta_family}

We now specialize the unit-cost process to iid
$\mathrm{Beta}(\alpha,1)$ draws.
This yields a one-parameter family of closed-form, pairwise-additive,
strictly proper scores.
The parameter $\alpha$ controls how concentrated the unit costs are near
$1$: large $\alpha$ makes the score increasingly rank-like, while small
$\alpha$ spreads mass toward $0$ and produces a smoother,
more magnitude-sensitive penalty.

Because $\mathrm{Beta}(\alpha,1)$ is continuous, the tie event in
\cref{eq:directed_pairwise_loss} has probability zero throughout this
subsection.

For $\alpha>0$, define
\begin{equation}
\label{eq:beta_shape}
  L_\alpha(r)
  \;:=\;
  \begin{cases}
    1+\left(1+\tfrac1\alpha\right)\!\left(1-r^\alpha\right),
      & r \le 1, \\[4pt]
    r^{-(\alpha+1)},
      & r > 1.
  \end{cases}
\end{equation}
and
\begin{equation}
\label{eq:beta_b}
  b_\alpha
  \;:=\;
  \frac{1}{\left(1+\tfrac1\alpha\right)\left(2+\tfrac1\alpha\right)},
\end{equation}

\begin{proposition}[Beta pairwise loss]
\label{prop:beta_pairwise}
Let $u_i,u_j \overset{\mathrm{iid}}{\sim} \mathrm{Beta}(\alpha,1)$ with
$\alpha>0$.
Then, for every $r>0$,
\begin{equation}
\label{eq:beta_pairwise_raw}
  \E\!\left[
    u_j\,\ind\!\left(u_j<\frac{u_i}{r}\right)
  \right]
  \;=\;
  b_\alpha\,L_\alpha(r).
\end{equation}
\end{proposition}

\begin{proof}
Write $f_\alpha(x)=\alpha x^{\alpha-1}$ for the density of
$\mathrm{Beta}(\alpha,1)$ on $(0,1)$.
For $t\in[0,1]$, define the truncated first moment
\[
  m_\alpha(t)
  \;:=\;
  \E\!\left[u_j\,\ind(u_j<t)\right]
  \;=\;
  \int_0^t y f_\alpha(y)\,dy
  \;=\;
  \frac{\alpha}{\alpha+1}t^{\alpha+1}.
\]
Conditioning the truncation threshold to $u_i$ gives
\[
  \E\!\left[
    u_j\,\ind\!\left(u_j<\frac{u_i}{r}\right)
  \right]
  \;=\;
  \E\!\left[
    m_\alpha\!\left(\min\!\left\{1,\frac{u_i}{r}\right\}\right)
  \right]
  \;=\;
  \frac{\alpha^2}{\alpha+1}
  \int_0^1
  x^{\alpha-1}
  \min\!\left\{1,\frac{x}{r}\right\}^{\alpha+1}
  dx.
\]

If $r>1$, then $x/r<1$ for all $x\in(0,1)$, so
\[
  \E\!\left[
    u_j\,\ind\!\left(u_j<\frac{u_i}{r}\right)
  \right]
  \;=\;
  \frac{\alpha^2}{\alpha+1}\,r^{-(\alpha+1)}
  \int_0^1 x^{2\alpha}\,dx
  \;=\;
  \frac{\alpha^2}{(\alpha+1)(2\alpha+1)}\,r^{-(\alpha+1)}
  \;=\;
  b_\alpha\,r^{-(\alpha+1)}
\]

If $r\le 1$, split the integral at $x=r$:
\begin{align*}
  \E\!\left[
    u_j\,\ind\!\left(u_j<\frac{u_i}{r}\right)
  \right]
  &=
  \frac{\alpha^2}{\alpha+1}
  \left[
    \int_0^r x^{\alpha-1}\left(\frac{x}{r}\right)^{\alpha+1}dx
    +
    \int_r^1 x^{\alpha-1}dx
  \right] \\
  &=
  \frac{\alpha^2}{\alpha+1}
  \left[
    \frac{r^\alpha}{2\alpha+1}
    \;+\;
    \frac{1-r^\alpha}{\alpha}
  \right] \\
  &=
  \frac{\alpha^2}{(\alpha+1)(2\alpha+1)}
  \left[
  \frac{
    \alpha + (1+\alpha)(1-r^\alpha)
  }{\alpha}
  \right] \\
  &=
  b_\alpha
  \left[
    1+\left(1+\tfrac1\alpha\right)(1-r^\alpha)
  \right]
\end{align*}

Combining the two cases gives $b_\alpha\,L_\alpha(r)$ as claimed.
\end{proof}

\begin{figure}[h]
  \centering
  \includegraphics[width=0.75\textwidth]{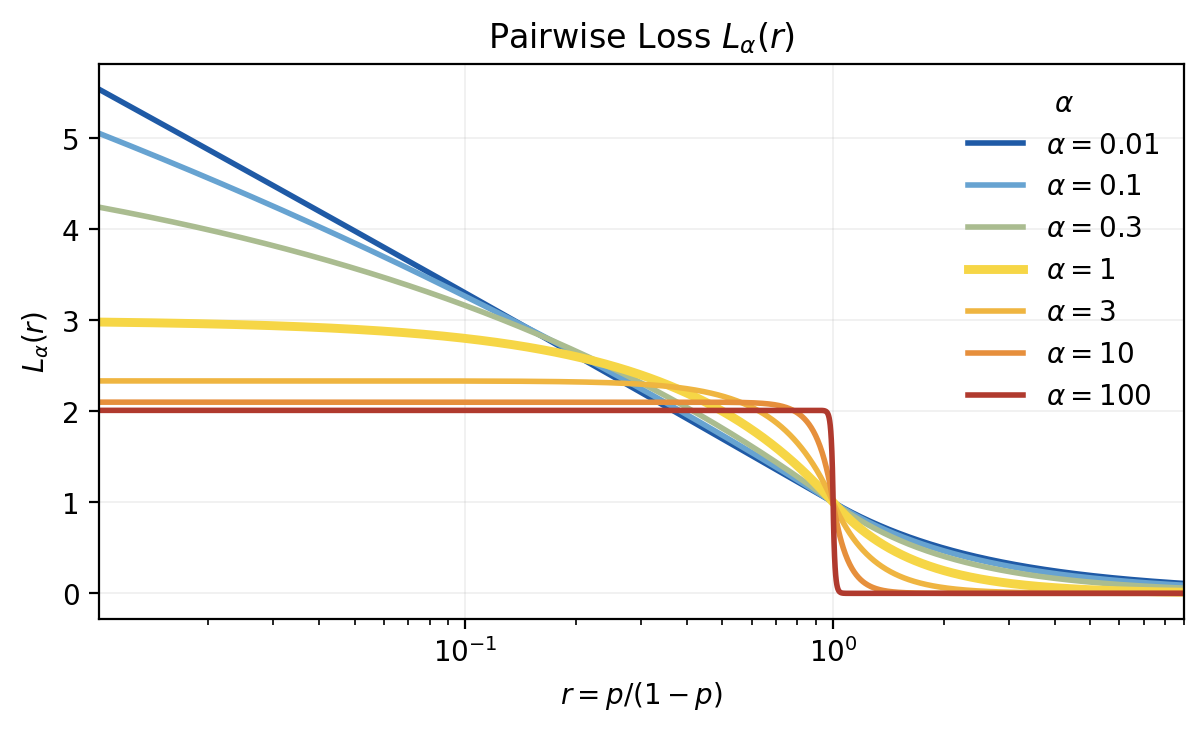}
  \caption{
    Normalized pairwise loss $L_\alpha(r)$ as a function of the odds ratio $r = p_i/p_j$ for several values of $\alpha$.
    All curves pass through $L_\alpha(1)=1$.
    As $\alpha\to\infty$, the loss converges to a step function (pure rank sensitivity).
    As $\alpha\to 0$, the left branch becomes logarithmic, penalizing underconfident forecasts more broadly.
  }
  \label{fig:gradient_beta}
\end{figure}

\begin{remark}[Gradient profile]
  \label{rmk:beta_gradient}
  Assume a softmax parameterization and write $\delta = z_i - z_j$ for the logit
  gap between the true class $i$ and a distractor $j$.
  Since $p_i/p_j = e^\delta$, the gradient of the pairwise loss term
  $L_\alpha(p_i/p_j)$ with respect to the distractor logit $z_j$ is
  \[
    \frac{\partial}{\partial z_j}
    L_\alpha\!\left(e^{z_i-z_j}\right)
    \;=\;
    -\frac{d}{d\delta} L_\alpha\!\left(e^\delta\right)
    \;=\;
    (\alpha+1)
    \begin{cases}
      e^{\alpha\delta}, & \delta \le 0, \\[4pt]
      e^{-(\alpha+1)\delta}, & \delta > 0.
    \end{cases}
  \]
  The gradient peaks when the true and distractor logits coincide ($\delta=0$) and decays exponentially on both sides, at rate $\alpha$ when the true class is behind ($\delta < 0$) and $\alpha+1$ when it is ahead ($\delta > 0$).
  Small $\alpha$ flattens the gradient for false classes ahead of the true class, while still decaying for false classes behind the true class.
  Large $\alpha$ steepens both decays, concentrating learning signal on near ties.
\end{remark}

\begin{corollary}[Weighted Beta score]
  \label{cor:beta_weighted}
  Let $c_k=C_k u_k$ with $C_k>0$ fixed and
  $u_k \overset{\mathrm{iid}}{\sim} \mathrm{Beta}(\alpha,1)$, where
  $\alpha>0$.
  Define
  \begin{equation}
  \label{eq:weighted_beta_score}
    S^{\mathrm{Beta}}_{\alpha,\mathbf C}(\vec p,i)
    \;:=\;
    \sum_{j\neq i}
    C_j\,L_\alpha\!\left(\frac{p_i/C_i}{p_j/C_j}\right).
  \end{equation}
  Then the expected search cost is
  \begin{equation}
  \label{eq:weighted_beta_affine}
    S(\vec p,y{=}i)
    \;=\;
    \frac{\alpha}{\alpha+1}\,C_i
    \;+\;
    b_\alpha\,S^{\mathrm{Beta}}_{\alpha,\mathbf C}(\vec p,i).
  \end{equation}
  Consequently, for every finite $\alpha>0$ and every choice of $C_k>0$,
  $S^{\mathrm{Beta}}_{\alpha,\mathbf C}$ is strictly proper.
  Equivalently, the expected search cost in
  \cref{eq:weighted_beta_affine} is strictly proper.
\end{corollary}
  
\begin{proof}
  By \cref{prop:known_cost_rescaling},
  \[
    S(\vec p,y{=}i)
    \;=\;
    C_i\,\E[u_i]
    \;+\;
    \sum_{j\neq i}
    C_j\,L_u^{i\to j}\!\left(\frac{p_i/C_i}{p_j/C_j}\right).
  \]
  Under iid $\mathrm{Beta}(\alpha,1)$ unit costs,
  \cref{prop:beta_pairwise} gives
  \[
    L_u^{i\to j}(r)=b_\alpha\,L_\alpha(r),
  \]
  and $\E[u_i]=\alpha/(\alpha+1)$.
  Substituting these identities yields
  \cref{eq:weighted_beta_affine}.
  
  For strict propriety, note that for every finite $\alpha>0$,
  the ratio $u_i/u_j$ has support $(0,\infty)$.
  Hence \cref{prop:known_cost_rescaling} implies that the expected search cost
  $S(\vec p,y{=}i)$ is strictly proper.
  Finally, \cref{eq:weighted_beta_affine} shows that this expected search cost
  is a positive affine transform of
  $S^{\mathrm{Beta}}_{\alpha,\mathbf C}$, so the two have the same unique
  minimizer in $\vec p$.
  Therefore $S^{\mathrm{Beta}}_{\alpha,\mathbf C}$ is strictly proper.
\end{proof}

\subsubsection{Three anchor regimes: \texorpdfstring{$\alpha=1$}{alpha=1}, \texorpdfstring{$\alpha\downarrow 0$}{alpha->0}, and \texorpdfstring{$\alpha\to\infty$}{alpha->infty}.}

To keep the two objects distinct, recall from
\cref{eq:weighted_beta_score,eq:weighted_beta_affine} that the
\emph{normalized} Beta score is
\[
  S^{\mathrm{Beta}}_{\alpha,\mathbf C}(\vec p,i)
  \;:=\;
  \sum_{j\neq i}
  C_j\,L_\alpha\!\left(\frac{p_i/C_i}{p_j/C_j}\right),
\]
while the corresponding \emph{raw} expected search cost satisfies
\[
  S(\vec p,y{=}i)
  \;=\;
  \frac{\alpha}{\alpha+1}\,C_i
  \;+\;
  b_\alpha\,S^{\mathrm{Beta}}_{\alpha,\mathbf C}(\vec p,i),
\]
where \(b_\alpha\) is defined in \cref{eq:beta_b}. The case \(\alpha=1\)
is an ordinary member of the Beta family, whereas \(\alpha\downarrow 0\)
and \(\alpha\to\infty\) are limiting regimes.

\begin{corollary}[Uniform unit-cost case: \texorpdfstring{$\alpha=1$}{alpha=1}]
\label{cor:beta_one}
For \(\alpha=1\),
\[
  L_1(r)
  \;=\;
  \begin{cases}
    3-2r, & r \le 1,\\[4pt]
    r^{-2}, & r > 1.
  \end{cases}
\]
Hence
\begin{equation}
\label{eq:beta_alpha1_affine}
  S(\vec p,y{=}i)
  \;=\;
  \frac12\,C_i
  \;+\;
  \frac16\,S^{\mathrm{Beta}}_{1,\mathbf C}(\vec p,i).
\end{equation}
In particular, \(S^{\mathrm{Beta}}_{1,\mathbf C}\) is strictly proper.
\end{corollary}

\begin{proof}
Set \(\alpha=1\) in \cref{eq:beta_shape,eq:weighted_beta_affine}. Then
\(b_1=\tfrac16\) and \(\alpha/(\alpha+1)=\tfrac12\), which yields
\cref{eq:beta_alpha1_affine}. For any true label distribution
\(\vec\pi\), the Bayes risk of \(S^{\mathrm{Beta}}_{1,\mathbf C}\) is
obtained from the Bayes risk of the raw expected search cost by
subtracting \(\tfrac12\sum_i \pi_i C_i\) and multiplying by \(6\). Hence
the minimizer in \(\vec p\) is unchanged. Since the raw expected search
cost is strictly proper for every finite \(\alpha>0\),
\(S^{\mathrm{Beta}}_{1,\mathbf C}\) is strictly proper as well.
\end{proof}

\begin{corollary}[Unit-cost specialization: Pandora's Regret]
\label{cor:pandora_regret}
Suppose \(C_k\equiv 1\). Define
\begin{equation}
\label{eq:pandora_from_beta}
  S_{\mathrm{Pandora}}(\vec p,i)
  \;:=\;
  \frac{1}{3(K-1)}
  \sum_{j\neq i}
  L_1\!\left(\frac{p_i}{p_j}\right).
\end{equation}
Then
\begin{equation}
\label{eq:pandora_affine}
  S(\vec p,y{=}i)
  \;=\;
  \frac{1+(K-1)\,S_{\mathrm{Pandora}}(\vec p,i)}{2}.
\end{equation}
In particular, \(S_{\mathrm{Pandora}}\) is strictly proper.
\end{corollary}

\begin{proof}
When \(C_k\equiv 1\),
\[
  S^{\mathrm{Beta}}_{1,\mathbf 1}(\vec p,i)
  =
  \sum_{j\neq i} L_1\!\left(\frac{p_i}{p_j}\right)
  =
  3(K-1)\,S_{\mathrm{Pandora}}(\vec p,i)
\]
by \cref{eq:pandora_from_beta}. Substituting this into
\cref{eq:beta_alpha1_affine} gives \cref{eq:pandora_affine}. As above,
this removes an additive term independent of \(\vec p\) and rescales by a
positive constant, so \(S_{\mathrm{Pandora}}\) is strictly proper.
\end{proof}

For the lower-endpoint limit, define the limiting pairwise loss
\begin{equation}
\label{eq:beta_zero_shape}
  L_0(r)
  \;:=\;
  \lim_{\alpha\downarrow 0} L_\alpha(r)
  \;=\;
  \begin{cases}
    1-\ln r, & r \le 1,\\[4pt]
    r^{-1}, & r > 1,
  \end{cases}
\end{equation}
and the corresponding limiting normalized score
\begin{equation}
\label{eq:beta_zero_score}
  S^{\mathrm{Beta}}_{0,\mathbf C}(\vec p,i)
  \;:=\;
  \sum_{j\neq i}
  C_j\,L_0\!\left(\frac{p_i/C_i}{p_j/C_j}\right).
\end{equation}
This is a limit score, not a finite-\(\alpha\) Beta score.

\begin{corollary}[Lower-endpoint limit: \texorpdfstring{$\alpha\downarrow 0$}{alpha->0}]
\label{cor:beta_zero}
For every \(\vec p\) and class \(i\),
\[
  S^{\mathrm{Beta}}_{\alpha,\mathbf C}(\vec p,i)
  \;\longrightarrow\;
  S^{\mathrm{Beta}}_{0,\mathbf C}(\vec p,i)
  \qquad\text{as }\alpha\downarrow 0.
\]
Moreover, \(S^{\mathrm{Beta}}_{0,\mathbf C}\) is strictly proper.
\end{corollary}

\begin{proof}
  For \(r>1\), \cref{eq:beta_shape} gives \(L_\alpha(r)=r^{-(\alpha+1)}\to r^{-1}\).
  For \(r\le 1\), the expansion \(r^\alpha=1+\alpha\ln r+O(\alpha^2)\) gives \(L_\alpha(r)\to 1-\ln r\).
  This proves \cref{eq:beta_zero_shape}, and termwise convergence gives \(S^{\mathrm{Beta}}_{\alpha,\mathbf C}(\vec p,i) \to S^{\mathrm{Beta}}_{0,\mathbf C}(\vec p,i)\).
  
  For strict propriety, note that \(L_0\) is continuous and strictly decreasing on \((0,\infty)\), so the pairwise Bayes-risk argument of \cref{thm:properness} applies directly.
\end{proof}

\begin{corollary}[Deterministic rank limit: \texorpdfstring{$\alpha\to\infty$}{alpha->infty}]
\label{cor:beta_infty}
Write
\[
  s_k:=\frac{p_k}{C_k},
\]
and recall \(h(a,b):=\ind(a>b)+\tfrac12\ind(a=b)\). Then
\[
  L_\alpha(r)\;\longrightarrow\;2\,h(1,r)
  \;=\;
  \begin{cases}
    2, & r<1,\\[2pt]
    1, & r=1,\\[2pt]
    0, & r>1.
  \end{cases}
\]
Therefore the normalized score satisfies
\begin{equation}
\label{eq:beta_infty_normalized}
  S^{\mathrm{Beta}}_{\alpha,\mathbf C}(\vec p,i)
  \;\longrightarrow\;
  2\sum_{j\neq i} C_j\,h(s_j,s_i).
\end{equation}
And since \(b_\alpha\to \tfrac12\) in \cref{eq:weighted_beta_affine}, the
raw expected search cost satisfies
\begin{equation}
\label{eq:beta_infty_raw}
  S(\vec p,y{=}i)
  \;\longrightarrow\;
  C_i+\sum_{j\neq i} C_j\,h(s_j,s_i).
\end{equation}
Thus the \(\alpha\to\infty\) limit retains only the weighted rank
ordering of the scores \(s_k\), and the limiting normalized score is not
strictly proper.
\end{corollary}

\begin{proof}
  If \(r<1\), then \(r^\alpha\to 0\) in \cref{eq:beta_shape}, so \(L_\alpha(r)\to 2\).
  If \(r>1\), then \(r^{-(\alpha+1)}\to 0\).
  And \(L_\alpha(1)=1\) for every \(\alpha\).
  Therefore, \(L_\alpha(r)\to 2\,h(1,r)\).
  Substituting \(r=s_i/s_j\) proves \cref{eq:beta_infty_normalized},
  and \cref{eq:beta_infty_raw} follows from \cref{eq:weighted_beta_affine} using \(\alpha/(\alpha+1)\to 1\) and \(b_\alpha\to\tfrac12\).
  The limiting score depends on \(\vec p\) only through the ranking of \(s_1,\dots,s_K\), so its Bayes risk is constant on open regions of the simplex where that ranking is unchanged.
  Therefore the minimizer is not unique and strict propriety fails.
\end{proof}


\section{Log Loss as Search Cost}
\label{apdx:logloss_search}

We prove that interpreting log loss as a search problem reveals a strict mathematical "cost amnesia," in which the metric perfectly erases all economic penalties for early wrong turns.
In a sequential search model (\cref{thm:cost_amnesia}), the expected overtreatment costs telescope multiplicatively to a constant, regardless of the evaluation order.
In a parallel decision model (\cref{thm:cost_amnesia_parallel}), they collapse additively via the probability simplex.
Under both architectures, the total expected cost reduces strictly to $-\ln p_m$ plus a universal baseline.
This demonstrates that blindness to search efficiency is not a quirk of a specific decision pipeline, but an intrinsic structural pathology of the logarithm interacting with the probability simplex.

\subsection{Cost Amnesia in Fixed-Order Search}
\label{apdx:cost_amnesia_proof}

We model multiclass evaluation as a sequence of binary rule-in decisions in a
fixed order $\sigma$.
Let
\[
S_{t-1}=\sum_{s=t}^K p_{\sigma(s)},
\qquad
h_t=\frac{p_{\sigma(t)}}{S_{t-1}}.
\]
At step $t$, the current class has conditional probability $h_t$ among the
remaining mass.

To connect with literal search costs, suppose acting on the current class costs
$c$, while deferring to the true class costs $1$.
If the policy acts whenever $h\ge c$, the realized step loss is
\[
  \ell_c(h)
  = c\,\ind(h\ge c)
  + \ind(\sigma(t)=i,\; h<c).
\]
The ex post optimal action incurs loss $c\,\ind(\sigma(t)=i)$, so the stepwise regret is
\[
  \ell_c(h)-c\,\ind(\sigma(t)=i)
  = c\,\ind(\sigma(t)\neq i,\; h\ge c)
  + (1-c)\,\ind(\sigma(t)=i,\; h<c).
\]
We integrate this regret against the Haldane measure
\[
d\mu(c)=\frac{dc}{c(1-c)}.
\]

\begin{theorem}[Cost Amnesia for Fixed-Order Search]
\label{thm:cost_amnesia}
Let $\vec{p}\in\Delta^{K-1}$ and let $i$ be the true class, appearing at position $m$ in the fixed search order $\sigma$. Then the
total integrated regret of the sequential search is
\[
\mathcal R_\sigma(\vec{p},\,i)=-\ln p_i.
\]
\end{theorem}

\begin{proof}
At each false step $t<m$ (where $\sigma(t)\neq i$), only the overtreatment term contributes:
\[
J_t^{\text{false}}
=
\int_0^{h_t} c\,\frac{dc}{c(1-c)}
=
-\ln(1-h_t)
=
-\ln\frac{S_t}{S_{t-1}}.
\]
At the true step $t=m$ (where $\sigma(t)=i$), only the undertreatment term contributes:
\[
J_m^{\text{true}}
=
\int_{h_m}^1 (1-c)\,\frac{dc}{c(1-c)}
=
-\ln h_m
=
-\ln\frac{p_i}{S_{m-1}}.
\]
Therefore
\[
\mathcal R_\sigma(\vec{p},\,i)
=
J_m^{\text{true}}+\sum_{t=1}^{m-1}J_t^{\text{false}}
=
-\ln\frac{p_i}{S_{m-1}}
-\sum_{t=1}^{m-1}\ln\frac{S_t}{S_{t-1}}
=
-\ln p_i. \qedhere
\]
\end{proof}

\begin{remark}[Raw cost formulation]
If one integrates the raw search cost rather than regret in the literal model
with action cost $c$ and miss cost $1$, the positive-class term requires
truncation near $c=1$.
After subtracting the resulting order-independent baseline $C(\epsilon)=\ln(1{-}\epsilon)-\ln\epsilon$ and letting
$\epsilon\downarrow 0$, one recovers the same limit $-\ln p_i$.
We state the theorem in regret form because it preserves the $(c,1)$ search
semantics while avoiding truncation.
\end{remark}

\begin{remark}[The Haldane Knife-Edge]
\label{rmk:haldane_critical}
Consider the family $w(c) = \alpha/c + \beta/(1{-}c)$. For a false-class step, expected overtreatment cost evaluates to $(\alpha - \beta)\,h_t - \beta\ln(1 - h_t)$. The logarithmic term telescopes cleanly via the chain rule for any $\beta$. The linear residual $(\alpha - \beta)\,h_t$, however, does not. If $\alpha > \beta$ (prior mass skewed toward cheap tests), the residual is positive; the metric implicitly rewards testing \emph{unlikely} classes first. If $\alpha < \beta$, it is negative, rewarding a greedy strategy of testing likely classes first. The Haldane prior ($\alpha = \beta$) is the unique knife-edge where the residual vanishes, forcing the metric to be completely blind to search economics.
\end{remark}

\begin{remark}[Why this is not just Shannon coding]
  \label{rmk:shannon_caveat}
  The identity \[
  -\ln p_i = -\ln h_m - \sum_{t=1}^{m-1}\ln(1-h_t)
  \] is a chain-rule decomposition of the probability assigned to the true class.
  This resembles the pathwise decomposition of code length in a binary decision tree, but in fact in source coding, the questioner designs each query to balance probabilities and maximize information per step, and the step costs are constant.
  In diagnostic search, however, the available tests are fixed and their costs are heterogeneous.
  The core of our result is to show that under the Haldane-weighted threshold-regret representation with fixed ordering and a constrained action set, the resulting score collapses to the same logarithmic form.
\end{remark}

\subsection{Removing Sequentiality: The Parallel Case}
\label{apdx:parallel_proof}

\begin{theorem}[Cost Amnesia---Parallel]
\label{thm:cost_amnesia_parallel}
Under $K$ simultaneous, independent binary decisions evaluated with the scale-invariant measure $dc/c$ on $(0, 1]$ and a fixed miss cost of~$1$, the total integrated cost is:
\begin{equation}
  \mathcal{L}_\parallel = 1 - \ln p_i
\end{equation}
where $i$ is the true class.
\end{theorem}

\begin{proof}
Let the optimal policy treat class~$k$ if and only if $c_k < p_k$. No truncation is required: although $dc/c$ has a pole at zero, the linear action cost~$c$ absorbs it at every step.

\medskip\noindent\textbf{True class ($k = i$).}
The risk combines the cost of correct treatment with the cost of a miss:
\begin{equation}
  \mathcal{R}_i
  = \int_0^{p_i} c \left(\frac{1}{c}\right) dc
    + \int_{p_i}^1 1 \left(\frac{1}{c}\right) dc
  = p_i - \ln p_i.
\end{equation}

\medskip\noindent\textbf{False class ($k \neq i$).}
Correct rejection is free; only overtreatment contributes:
\begin{equation}
  \mathcal{R}_k
  = \int_0^{p_k} c \left(\frac{1}{c}\right) dc
  = p_k.
\end{equation}

\medskip\noindent\textbf{Total cost.}
Summing the independent decisions yields:
\begin{equation}
  \mathcal{L}_\parallel
  = \mathcal{R}_i + \sum_{k \neq i} \mathcal{R}_k
  = -\ln p_i
    + \underbrace{\sum_{k=1}^{K} p_k}_{=\,1}
  = 1 - \ln p_i. \qedhere
  \label{eq:total_parallel}
\end{equation}
\end{proof}

\begin{remark}
We have already seen this $dc/c$ measure as the $\alpha\to0$ limit of the Pandora family we define in \cref{sec:beta_family}. So it can be lifted to the sequential search setting with optimal ordering for each cost, but then we recover the Pandora score of \cref{apdx:beta_family}, not the log loss.
\end{remark}

\section{Top-1 Measures}

\begin{table}[h]
  \centering
  \small
  \begin{tabular}{p{0.10\textwidth}cp{0.35\textwidth}cp{0.35\textwidth}}
    \toprule
    Desideratum & \multicolumn{2}{c}{Accuracy} & \multicolumn{2}{c}{Macro F1} \\
    \cmidrule(lr){2-3} \cmidrule(lr){4-5}
    & & Assessment & & Assessment \\
    \midrule
    Decision alignment & $\times$
      & Minimizes $1{-}q_{\hat{k}}$ via $\arg\max$; ignores external costs.
      & $\times$
      & Non-decomposable \citep{narasimhan15}: thresholds driven by aggregate confusion counts ($\Delta\text{F1}_k^{\mathrm{TP}}$, $|\Delta\text{F1}_k^{\mathrm{FP}}|$), not clinical costs.
      \\[4pt]
    Within-instance order\linebreak sensitivity & $\times$
      & Only the top class matters; remaining $K{-}1$ ranking ignored.
      & $\times$
      & Also driven by $\arg\max$; remaining $K{-}1$ ranking ignored.
      \\[4pt]
    Strict\linebreak propriety & $\times$
      & Proper, but not strictly proper. Any $\vec{p}$ preserving $\arg\max$ is optimal; minimizer not unique.
      & $\times$
      & Optimal report depends on population-level counts, not solely on $\vec{q}$ \citep{proper07}.
      \\
    \bottomrule
  \end{tabular}
  \caption{Top-1 measures evaluated against the desiderata of \cref{sec:desiderata}.
    Both fail all three, but for different reasons:
    accuracy is cost-blind and flat over the simplex;
    Macro~F1 additionally violates instance-level independence
    by coupling predictions through endogenous, classifier-dependent thresholds.}
  \label{tab:top1_desiderata}
\end{table}

\subsection{Accuracy}
\label{apdx:accuracy_proofs}

Accuracy assigns the score $A(\vec{p}, i) = \ind(\arg\max_k p_k = i)$,
i.e.\ zero-one loss on the hard prediction (ties are measure zero, but typically broken by lowest index).
It is a \emph{decision metric}, not a scoring rule:
it evaluates hard classifications $\hat{y} \in \{e_1,\dots,e_K\}$,
not probability vectors $\vec{p} \in \Delta^{K-1}$.
We include it in the desiderata comparison because accuracy is routinely
used to rank probabilistic models,
and this ranking behavior is what our framework interrogates.

\subsection{The Structural Misalignment of Macro F1}
\label{apdx:f1_proofs}

F1 is a well-known information retrieval metric, extensively critiqued for ignoring true negatives and failing to represent any quantity in the real world \cite{christen24}.  We can write this expression for any class $k$.
$$\text{F1}_k = \frac{2\,\text{TP}_k}{2\,\text{TP}_k + \text{FP}_k + \text{FN}_k}$$
There are various extensions to multiclass and multilabel settings, but by far the most popular is Macro F1, the unweighted average of per-class F1 scores:
$$\text{F1}_{\text{macro}} = \frac{1}{K}\sum_{k=1}^{K}\text{F1}_k$$
We argue that it is structurally misaligned with sequential search
because it aggregates in the wrong direction, induces decision
costs with no clinical basis, and violates patient-level independence.

\subsubsection{The Aggregation Transposition}
\label{apdx:f1_transposition}

Sequential diagnostic search is \emph{patient-pivoted}:
for a single patient, the model must rank $K$ candidate diagnoses.
Macro~F1 inverts this orientation.
It aggregates \emph{per class} across patients
(\emph{category-pivoted} or \emph{label-based} evaluation,
in the terminology of \citet{sebastiani02} and \citet{tsoumakas10}).
The sequence that matters to Macro~F1 is the sequence in which examples contribute to classwise precision and recall, not the sequence in which alternative classes are considered within a single instance.
Because Macro~F1 aggregates horizontally across the population
whilst search operates vertically within a patient,
the metric is structurally blind to within-patient ranking.
It is in principle possible to use Instance F1, but this is rarely done, and \cite{resin23} shows that it is not proper either.

\subsubsection{F1 as Asymmetric Search with Endogenous Costs}
\label{apdx:f1_decision}

Macro~F1 is non-decomposable in the sense of \citet{narasimhan15}:
clinical decisions are mathematically coupled across the evaluation set.
Building on their marginal analysis,
we can characterize the greedy Macro~F1-optimal prediction
as a cost-sensitive decision rule.
The per-class rewards and penalties for true positives and false positives
are determined entirely by the classifier's own aggregate confusion
counts, not by any exogenous or clinically meaningful quantity.

\begin{tcolorbox}
  \begin{proposition}[Macro F1 as a Cost-Sensitive Decision Rule]
    \label{prop:f1_decision}
    In the single-label multiclass setting, given a true conditional class distribution~$\vec{p}$, the greedy Macro~F1-optimal prediction simplifies to the decoupled $\mathcal{O}(K)$ decision rule:
    \begin{equation}
    \label{eq:f1_argmax}
      \hat{k}
      \;=\;
      \arg\max_k\;\Bigl[\,
        p_k\,\bigl(\Delta\text{F1}_k^{\mathrm{TP}}
          + 2\,\bigl|\Delta\text{F1}_k^{\mathrm{FP}}\bigr|\bigr)
        - \bigl|\Delta\text{F1}_k^{\mathrm{FP}}\bigr|
      \,\Bigr],
    \end{equation}
    where $\Delta\text{F1}_k^{\mathrm{TP}}$ and $\Delta\text{F1}_k^{\mathrm{FP}}$ denote the discrete marginal changes to the $\mathrm{F1}$ score of class~$k$ from incrementing its true positives and false positives, respectively.
    \end{proposition}
    
    \begin{proof}
    Let $D_k = 2\,\mathrm{TP}_k + \mathrm{FP}_k + \mathrm{FN}_k$ denote the denominator of $\mathrm{F1}_k$.
    Following \citet{narasimhan15} but replacing continuous gradients with exact discrete marginals, the changes to $\mathrm{F1}_k$ from incrementing a single confusion-matrix entry are
    \begin{align}
      \Delta\text{F1}_k^{\mathrm{TP}}
      &= \frac{2(\mathrm{FP}_k + \mathrm{FN}_k)}{D_k\,(D_k + 2)},
      \label{eq:f1_tp_reward}
      \\[4pt]
      \bigl|\Delta\text{F1}_k^{\mathrm{FP}}\bigr|
      \;=\; \bigl|\Delta\text{F1}_k^{\mathrm{FN}}\bigr|
      &= \frac{2\,\mathrm{TP}_k}{D_k\,(D_k + 1)}.
      \label{eq:f1_fp_cost}
    \end{align}
    
    Predicting class~$k$ when the true class is~$j \sim \vec{p}$ yields a true positive if $j=k$ (probability $p_k$). Otherwise, if $j \neq k$ (probability $p_j$), it simultaneously forces a false positive on~$k$ and a false negative on~$j$. Summing over these mutually exclusive outcomes, the expected marginal change in $\text{F1}_{\mathrm{macro}}$ is
    \begin{align}
      \E\bigl[\Delta\text{F1}^{\mathrm{PP}}_k\bigr]
      &= p_k\,\Delta\text{F1}_k^{\mathrm{TP}}
      \;-\; \sum_{j \neq k} p_j \Bigl( \bigl|\Delta\text{F1}_k^{\mathrm{FP}}\bigr| + \bigl|\Delta\text{F1}_j^{\mathrm{FN}}\bigr| \Bigr) \notag \\[4pt]
      &= p_k\Bigl(
        \Delta\text{F1}_k^{\mathrm{TP}}
        \;+\; 
        \bigl|\Delta\text{F1}_k^{\mathrm{FN}}\bigr|
      \Bigr)
      \;-\; (1 - p_k)\,\bigl|\Delta\text{F1}_k^{\mathrm{FP}}\bigr|
      \;-\; \sum_j p_j\,\bigl|\Delta\text{F1}_j^{\mathrm{FN}}\bigr|.
      \label{eq:f1_marginal}
    \end{align}
    
    Because the sum $\sum_j p_j\,\bigl|\Delta\text{F1}_j^{\mathrm{FN}}\bigr|$ is a global constant independent of the predicted class~$k$, and because $\bigl|\Delta\text{F1}_k^{\mathrm{FN}}\bigr| = \bigl|\Delta\text{F1}_k^{\mathrm{FP}}\bigr|$ from~\eqref{eq:f1_fp_cost}, maximizing $\E\bigl[\Delta\text{F1}^{\mathrm{PP}}_k\bigr]$ over~$k$ directly reduces to the rule in~\eqref{eq:f1_argmax}.
    \end{proof}
\end{tcolorbox}

In the limit of large $D_k$, this becomes
$$\hat{k} = \arg\max_k\;\frac{1}{D_k}\!\left(p_k - \tfrac{\text{F1}_k}{2}\right)$$

This derivation reveals how the effective cost of evaluating a condition for a given patient depends on confusion counts over all other patients, violating patient-level independence.
If the model has historically performed poorly on class~$k$, the effective reward for a true positive on~$k$ inflates, incentivizing the model to favor that class to rehabilitate its macro-average.

\section{Sequential Evaluation Procedure}
\label{apdx:simulator}

This appendix specifies the reduced-form operational cost model used in our experiments and clarifies its relationship to the idealized theory in the main text.
The model should be interpreted as a hypothesis-prioritization abstraction for evaluation, rather than as a literal description of clinical workflow.
Given a forecast $\vec p(x)$, the simulator converts probabilities into an ordered sequence of diagnostic actions and accumulates costs until the true condition is identified.

\paragraph{Assumptions.}
In the reduced-form simulator, classes are mutually exclusive and tests are perfectly informative: evaluating the true class terminates the search, while evaluating an incorrect class incurs its cost and eliminates one hypothesis.

\paragraph{Setup \& Policy.}
For an input $x$ with true class $i$, a model predicts $\vec p(x) \in \Delta^{K-1}$. A cost vector $\vec c \in \mathbb{R}_{>0}^K$ assigns a strictly positive testing cost to each class (see \cref{app:costs}).
The same $\vec c$ is used for all models and all test-set instances within a given experiment; only the forecast $\vec p$ varies across models.
This fixed-$\vec c$ design defines a common operational environment for comparing models, rather than asserting that all cases share identical real-world diagnostic costs.
Under these assumptions, the optimal sequential policy evaluates classes in descending order of the ratio $p_k/c_k$ (by pairwise exchange; see \cref{apdx:search_scoring}), with ties broken by class index.

Let $\sigma$ denote the resulting search sequence and $t^*$ the step at which the true class $i$ is evaluated ($\sigma_{t^*} = i$). The realized search cost is
\[
C(\vec p,\,i;\,\vec c)
\;=\;\sum_{t=1}^{t^*} c_{\sigma_t}.
\]

\paragraph{Aggregation.}
Over a test set of $N$ instances $\{(x_n,i_n)\}_{n=1}^N$, the aggregate simulated cost for model $m$ is
\[
\widehat{J}_{\mathrm{sim}}(m)
\;=\;
\frac{1}{N} \sum_{n=1}^N C\!\bigl(\vec p_m(x_n),\,i_n;\,\vec c\bigr).
\]
We rank candidate models by $\widehat{J}_{\mathrm{sim}}$ and compute Kendall's $\tau$ between this ranking and the rankings induced by standard evaluation metrics.

\paragraph{Instance-varying costs.}
The fixed-$\vec c$ assumption is an evaluation simplification; in practice, diagnostic costs may vary by patient, presentation, or clinical setting.
The bilevel framework in \cref{sec:bilevel} relaxes this simplification by taking expectations over a distribution of cost vectors, thereby assessing robustness across plausible operational environments rather than under a single shared cost specification.

\subsection{The Multiplexing Abstraction}
\label{apdx:multiplexing}

The simulator is a hypothesis-prioritization abstraction, not a literal claim that clinicians run one physical test per class.
We interpret $c_k$ as the expected incremental diagnostic expenditure induced by prioritizing condition~$k$, including any downstream workup that this prioritization decision typically triggers in the relevant clinical pathway.

In practice, many diagnostic procedures are multiplexed: one biopsy may resolve several histopathological hypotheses, and one imaging study may inform multiple candidate retinal conditions.
The sequential model abstracts such bundled procedures into hypothesis-level prioritization costs.
Its purpose is therefore not to reproduce the physical ordering of assays, but to evaluate how well a model's ranking of hypotheses aligns with the cost-sensitive ordering induced by an idealized decision process.

\subsection{Robustness to Imperfect Tests and Treatment Costs}
\label{apdx:imperfect_tests}

Imperfect tests and treatment payoffs modify the effective cost of prioritizing a condition but, under the assumptions below, preserve the same ranking form.
Accordingly, the scoring rules from \cref{apdx:search_scoring,apdx:beta_family} extend after replacing per-class costs by the effective costs defined below.

\paragraph{Imperfect tests.}
Let $s_i$ denote the sensitivity of test~$i$ and $f_i$ its false-positive rate.
A false positive triggers a confirmatory workup costing~$C_i$ before the condition is ruled out and the search resumes.
Assuming conditionally independent test outcomes given the true class, \citet{kress08} show by pairwise exchange that the optimal policy is greedy in the index
\begin{equation}\label{eq:kress_index}
  I_i\;=\;\frac{s_i\,p_i}{c_i + f_i\,C_i}\,,
\end{equation}
the probability of resolving the case per unit of expected diagnostic expenditure (test cost plus expected false-positive workup).
Defining the \emph{effective cost}
$\tilde{c}_i := (c_i + f_i\,C_i)/s_i$
yields $I_i = p_i/\tilde{c}_i$.
Under conditional independence and the reduced-form assumptions above, this reparameterization places the imperfect-test setting in the same ordering family as the perfect-test model: low sensitivity inflates cost by $1/s_i$, while false positives add the expected confirmatory-workup burden $f_i C_i/s_i$.
Under the same conditions, the decomposition and propriety results from \cref{thm:decomposition} through \cref{cor:beta_infty} apply after substituting $c_i \mapsto \tilde{c}_i$.
Strict propriety requires full support of the effective cost ratio $\tilde{c}_i / \tilde{c}_j$, which holds when $c_i$ has unbounded upper support but can fail when both $c_i$ and $f_i\,C_i$ are bounded (since confirmatory workup then imposes a positive lower bound on $\tilde{c}_i$; see \cref{thm:properness}\,(ii)).

If test outcomes are correlated (e.g., histopathological findings from related tissue sites), the greedy index remains a useful approximation, but exact optimality need not hold; analyzing that setting requires a more general sequential-testing model \citep[see][]{kress08}.

\paragraph{False negatives and search restarts.}
Suppose that when the true class~$j$ is tested, it is missed with probability $1 - s_j$, after which the search proceeds through the remaining classes and restarts.
Under the maintained assumptions that failed passes incur the same total classwise cost regardless of within-pass order, and that restarts do not introduce path-dependent updates beyond elimination, the probability of failing to resolve the case in a given pass is $1 - s_j$, independent of the search order.
The expected number of passes until resolution is therefore $1/s_j$.
Each failed pass incurs a total cost that sums over all classes and does not depend on the ordering.
The only ordering-dependent term is the cost accumulated within the successful pass before the true class is identified, which has the same ordering structure as the perfect-test problem with per-class costs $c_i + f_i\,C_i$.
The multiplicative factor $1/s_j$ does not affect which ordering minimizes expected cost and is already absorbed into $\tilde{c}_i = (c_i + f_i\,C_i)/s_i$.
This parallels the rescue-time separation in Equation~(3) of \citet{kress08}: the restart overhead depends on the true class but not on the search order, so it factors out.

\paragraph{Treatment and disease costs.}
Treatment payoffs contribute a second order-invariant term provided they depend only on the resolved true condition and not on the diagnostic path or time to resolution.

Let the true condition~$j$ carry treatment benefit~$B_j$, treatment harm~$r_j$, and untreated-disease cost~$D_j$.
Upon resolution, a true positive (probability $s_j$) yields payoff $B_j - r_j$, while a false negative (probability $1{-}s_j$) incurs $D_j$.
The realized payoff
\[
  \lambda_j \;=\; s_j\,(B_j - r_j) \;-\; (1-s_j)\,D_j
\]
is a function of the true condition alone.
Under this assumption, averaging over $j$ adds an order-invariant constant to expected cost, so the optimal ordering and associated scoring rule still depend only on $\tilde{c}_i$.

\medskip
\noindent
In summary, the perfect-test, zero-treatment-cost model in the main text captures the core ordering structure under the reduced-form assumptions used here.
Imperfect tests and treatment payoffs modify the effective per-hypothesis costs and, under the stated independence and order-invariance conditions, preserve the same ranking form.
Outside those conditions (for example, with correlated tests or path-dependent utilities), the resulting index should be interpreted as an informed approximation rather than an exact characterization.

\clearpage
\section{Diagnostic Cost Specification}
\label{app:costs}

The cost vector $\mathbf{c} = (c_1, \ldots, c_K)$ represents the expected economic intensity of the diagnostic workup triggered to confirm or rule out condition $k$. 
In the Pandora framework, the decision-maker tests hypotheses sequentially. 

\paragraph{Clinical Setting \& Base Costs.} 
To reflect marginal costs in a real-world setting, we model two plausible clinical scenarios:
(1) A primary care office visit for a dermatologic complaint 
(2) A community optometry visit for comprehensive eye exam (assuming use of routine OCT screening)
Clinical pathways were defined to adhere where possible to National Cancer Center Network guidelines.
In each of these scenarios, positive algorithmic suspicion will result in further healthcare costs such as subspecialist visit and potential procedures, a strictly positive marginal diagnostic expense ($c_k > 0$).
For the PCP visit, positive algorithmic suspicion prompts referral to a dermatologist at a new patient 15-29 minute office visit (CPT 99202, \$75.15) inclusive of total body skin exam with potential for biopsy.
For the optometrist visit, positive algorithmic suspicion prompts referral to an ophthalmologist office with comprehensive dilated eye exam (CPT~92004, \$149.64) and potential additional staining.
All base estimates utilize the CY~2026 CMS national non-facility Physician Fee Schedule (conversion factor \$33.4009) \cite{cms2026pfs}.
Where a workup involves multiple billable services, we sum the exact national payment amounts.

\paragraph{MedMNIST.}
We evaluate on two multiclass MedMNIST datasets with clinically distinct downstream actions:
DermaMNIST (7 skin lesion classes) derived from the HAM10000 dataset \cite{tschandl18} and OCTMNIST (4 retinal conditions) derived from the Kermany OCT dataset \cite{kermany18}.
We assume a triage setting for each dataset.
We prioritize hypothesis triggers to a specialist dermatologist or ophthalmologist.
We restrict attention to settings where sequential diagnostic search is meaningful (more than two labels, task is feasible, labels map to treatment decisions).
The implications of this restriction are as follows for each dataset.

\paragraph{DermaMNIST.}
The DermaMNIST dataset consists of skin lesions representing seven categories: melanoma, basal cell carcinoma, actinic keratoses/intraepithelial carcinoma (AKIEC), vascular lesions (broad group ranging from cherry angiomas to hemorrhage to angiokeratomas), melanocytic nevi, benign keratoses (a combined subgroup of seborrheic keratoses or "senile wart", solar lentigo or "sunspot", and lichen-planus like keratoses (LPLK)), and dermatofibroma.
Among these, three diagnoses present opportunities for meaningful sequential diagnostic search.

Lesions suspicious for melanoma, a potentially fatal skin cancer, require an office visit to a dermatologist (CPT 99202, \$75.15) with wide excisional biopsy whose type may vary based on size and location of the lesion.
The most common biopsies are punch biopsy (CPT~11104, \$121.25) or excisional biopsy (CPT 11106, \$151.31).
For the purposes of our model we have averaged these two codes for net cost of \$136.28.
The biopsied tissue undergoes standard histopathologic examination (CPT~88305, \$70.14), resulting in a total incremental cost of \$281.57.
Increasingly, melanoma samples are also sent for additional immunohistochemical (IHC) and staining with one review noting up to 50\% of samples tested \cite{ojukwu24};
while standard IHC stains may contain numerous we model a standard two-antibody IHC panel (CPT~88342, \$110.22; CPT~88341, \$94.19).
In our model, half of the pathways will result in this additional staining bringing the total pathway cost to \$383.77. 

Basal cell carcinoma, a common sun-exposure mediated skin cancer, typically requires a dermatologist visit (CPT 99202, \$75.15) with shave biopsy (CPT 11102, \$95.53) along with standard histopathology (CPT~88305, \$70.14) to confirm the diagnosis.
This pathway totals \$240.82. 

The benign keratoses category represents a mixed group of conditions which include senile warts, sunspots and LPLK.
As the original authors note, "From a dermatoscopic view, lichen planus-like keratoses are especially challenging because they can show morphologic features mimicking melanoma and are often biopsied or excised."
Given this, we propose that roughly 10\% of this category will follow the same pathway as melanoma with a total marginal cost of \$106.01.

The remaining four categories (vascular lesions, melanocytic nevi, dermatofibroma, and actinic keratoses/intraepithelial carcinoma) will require a nonurgent dermatologist visit for comprehensive skin exam alone.
The pathway cost is thus strictly the dermatologist CPT 99202 new patient visit fee (\$75.15).

\paragraph{OCTMNIST.}
The OCTMNIST dataset includes four categories of ocular pathologic features: drusen, choroidal neovascularization, diabetic macular edema, normal. 

Drusen are lipid-rich deposits which may represent simple aging or early age-related macular degeneration (AMD) depending on their characteristics, size and number \cite{fleckenstein24}.
If detected as part of comprehensive routine screening, this would trigger referral to ophthalmologist for comprehensive eye exam (CPT 92004, \$149.64) with color fundus photography (CPT~92250, \$37.07) in addition to likely repeat OCT for monitoring, totaling \$186.71. 

Detection of choroidal neovascularization (CNV), the pathologic growth of new blood vessels in the subretinal space commonly associated with more advanced age-related macular degeneration (AMD), would require an urgent referral to ophthalmologist for comprehensive dilated eye exam (CPT 92004), color fundoscopy (CPT 92250 \$37.07) plus fluorescein angiography (FA; CPT~92235, \$162.33) or increasingly OCT-angiography (CPT 92137, \$59.79) \cite{ly19,fleckenstein24}.
A multimodal approach is required to characterize the neovascular membrane and determine whether treatments such as anti-vessel growth therapy injections are warranted.
For the purposes of the model we will assume only 5\% of patients obtain OCT-A instead of fluorescein angiography in practice.
The marginal cost of this pathway thus totals \$343.91.

Diabetic macular edema is swelling due to a dysregulated blood-retinal barrier that can be a vision-threatening complication of diabetes.
If suspected, it requires urgent referral to an ophthalmologist for comprehensive eye exam (CPT 92004, \$149.64) with color fundus photography (CPT~92250, \$37.07) to characterize concomitant diabetic retinopathy.
Similarly to CNV, current gold standard involves use of FA to map vascular leakage and identify treatable lesions with the rise of OCT-A as a potential less invasive replacement.
Thus, this pathway also yields a marginal cost of \$343.91. 

Normal retinae require only the comprehensive clinical examination to confirm the remote screening result and definitively exclude peripheral pathology not visible on OCT.
These will likely only require annual follow-up in primary clinic for which we assign the 92004 \$149.64 exam fee.

\begin{table}[h]
  \centering
  \caption{Diagnostic evaluation cost vectors.
    Each $c_k$ represents the exact expected cost of the clinical pathway triggered to confirm or rule out hypothesis $k$ following a remote/triage screening.
    Derived from CY~2026 CMS non-facility RVUs.}
  \label{tab:cost_vector}
  \small
  \begin{tabular}{llrl}
    \toprule
    Dataset & Class & Cost (\$) & Modeled Clinical Pathway \\
    \midrule
    \multirow{7}{*}{DermaMNIST}
      & Actinic keratosis   & 75.15 & Visit only \\
      & Basal cell carc.    & 240.82 & Visit + shave bx + path \\
      & Benign keratosis    & 106.01 & Visit + 10\% melanoma pathway \\
      & Dermatofibroma      & 75.15  & Visit only \\
      & Melanoma            & 383.77 & Visit + bx + path + 50\% IHC panel \\
      & Melanocytic nevi    & 75.15  & Visit only \\
      & Vascular lesion     & 75.15  & Visit only \\
    \midrule
    \multirow{4}{*}{OCTMNIST}
      & CNV                 & 343.91 & Dilated exam + fundus photo + FA (5\% OCT-A) \\
      & DME                 & 343.91 & Dilated exam + fundus photo + FA (5\% OCT-A) \\
      & Drusen              & 186.71 & Dilated exam + fundus photo \\
      & Normal              & 149.64 & Dilated exam only \\
    \bottomrule
  \end{tabular}
\end{table}

\paragraph{FractureMNIST3D.}
\label{apdx:fracture_exclusion}

We exclude FractureMNIST3D from the main evaluation for two reinforcing reasons.

\paragraph{Unresolvable at $28^3$.}
FractureMNIST3D downsamples $64\,\text{mm}^3$ CT crops to $28\times28\times28$ voxels (${\sim}2.3\,\text{mm}$ isotropic), well below the resolution needed to distinguish cortical disruption patterns.
Accordingly, all published baselines perform poorly: the best reported AUC is $0.725$ and accuracy $50.8\%$ (chance $= 33.3\%$) \cite{medmnistv2}, and the dataset has been excluded from at least one foundation-model study on exactly these grounds \cite{schafer24}.
Model rankings derived from near-chance classifiers are dominated by noise, making any metric comparison (Pandora or otherwise) unreliable.

\paragraph{Labels do not map to treatment costs.}
The buckle/nondisplaced/displaced labels, inherited from RibFrac \cite{jin2020ribfrac}, describe single-fracture morphology.
Clinical management of rib fractures depends on patient-level factors rather than the displacement category of an isolated fracture crop:
total fracture count, flail chest, bilateral involvement, complication risk \cite{brasel2017}.
Assigning episode-of-care costs to these labels is therefore not clinically defensible.

\section{Training Details}
\label{apdx:training}

\paragraph{Model zoo.}
We evaluate pretrained image encoders from the TIMM library \cite{wightman19}.
To make the comparison deterministic and computationally manageable, we filter the pretrained TIMM models to those that
(i) accept 3-channel $224\times224$ inputs and
(ii) have at most 50M parameters.
In the \texttt{timm} version used for our experiments, this yields 601 compatible models; four of them emit non-finite features when run as frozen encoders on initialization, so we exclude them.

\paragraph{Datasets and splits.}
We use DermaMNIST and OCTMNIST from MedMNIST+ \cite{medmnistv1,medmnistv2}.
We use the $224\times224$ MedMNIST+ variants rather than the default $28\times28$ versions.
The data is downloaded automatically through the \texttt{medmnist} Python package.
Inputs are normalized with the ImageNet mean and standard deviation.
The single-channel OCTMNIST images are expanded to three channels by replication so that they are compatible with the TIMM zoo.
For both datasets, we report on the standard held-out test split.
We cap the training split at 10{,}000 images, applied via a deterministic seeded random subsample.
This leaves DermaMNIST's 7{,}007-image training split intact and reduces OCTMNIST (97{,}477 training images) to roughly 10\% of its training data, which keeps frozen-encoder feature extraction tractable across the full model zoo.

\paragraph{Linear-probe training protocol.}
For every pretrained backbone, we freeze the encoder and train only a linear classification head.
We extract features once per split by running the frozen encoder with average pooling (falling back to each architecture's default pool for the few backbones that reject average pooling) under fp16 mixed-precision inference, and cache the resulting feature tensors.
The feature dimension is measured from an actual forward pass rather than read from the model's reported \texttt{num\_features}, which is incorrect for some architectures.
The linear head is then trained on the cached features with cross-entropy loss using AdamW with learning rate $10^{-3}$, weight decay $10^{-4}$, batch size 64, and 1 epoch.
The code supports validation accuracy-based best-epoch selection, but we do not use it.
We use the same fixed training configuration for all models in the zoo rather than tuning hyperparameters separately for each backbone, since the goal of the experiment is to compare induced model rankings under a common protocol.

\paragraph{Reproducibility.}
All random seeds are fixed to 42 through a shared \texttt{set\_seed()} utility that seeds Python, NumPy, and PyTorch (including CUDA RNGs when available).
Together with the deterministic model-zoo filtering described above, this makes the experimental pipeline reproducible up to the usual nondeterminism of the underlying hardware and software stack.

\paragraph{Licenses for existing assets.}
We use DermaMNIST and OCTMNIST from the MedMNIST v2 benchmark \cite{medmnistv2}.
According to the MedMNIST project, MedMNIST is released under CC BY 4.0 except for DermaMNIST, which is released under CC BY-NC 4.0; thus, the MedMNIST subsets used here are CC BY-NC 4.0 for DermaMNIST and CC BY 4.0 for OCTMNIST.
DermaMNIST derives from HAM10000 \cite{tschandl18} (CC BY-NC-SA 4.0), and OCTMNIST derives from the Kermany OCT dataset \cite{kermany18} (CC BY 4.0).
Pretrained encoders are drawn from the PyTorch Image Models (TIMM) library \cite{wightman19}, which is released under Apache 2.0.
Because pretrained checkpoints in TIMM originate from multiple upstream sources, the applicable license for a given pretrained weight file is the license listed for that model in the TIMM repository.

\paragraph{Compute resources.}
Filtering approximately 600 pretrained models, extracting frozen features, training linear probes, and evaluating the ranking metrics takes approximately one day on a single GPU.
To reduce wall-clock time, we sharded part of the workload across two Google Colab GPU runtimes.
In our experiments, runs were performed on Google Colab GPU runtimes; both T4 and A100 were used.
The dominant cost is frozen-encoder feature extraction and pretrained-weight download rather than the optimization of the linear head.

\end{document}